 \documentclass[final,3p,times]{elsarticle}

\usepackage{amssymb}

\usepackage[ruled,vlined]{algorithm2e}
\usepackage{amssymb,amsfonts,amsmath}
\usepackage{enumerate}
\usepackage{proof}
\usepackage{color}
\usepackage{amsthm}

\usepackage[config]{subfig}

\usepackage{amssymb,amsfonts,amsmath}
\usepackage{enumerate}
\usepackage{subfig}
\usepackage{float}
\usepackage{caption}

\usepackage{hyperref}
\newtheorem{Theorem}{Theorem}

\newtheorem{Lemma}{Lemma}

\newtheorem{Proposition}{Proposition}

\newcommand{\field}[1]{\mathbf{#1}}

\newcommand{\bb}{\field{b}}
\newcommand{\ba}{\field{a}}

\newcommand{\Z}{\field{Z}}

\newcommand{\C}{\field{C}}
\newcommand{\bc}{\field{c}}
\newcommand{\B}{\field{B}}

\newcommand{\X}{\field{X}}
\newcommand{\x}{\field{x}}
\newcommand{\p}{\field{p}}
\newcommand{\s}{\field{s}}
\newcommand{\Sb}{\field{S}}
\newcommand{\Y}{\field{Y}}
\newcommand{\y}{\field{y}}
\newcommand{\w}{\field{w}}
\newcommand{\I}{\field{I}}
\newcommand{\m}{\field{m}}
\newcommand{\M}{\field{M}}
\newcommand{\V}{\field{V}}
\newcommand{\D}{\field{D}}
\newcommand{\W}{\field{W}}

\newcommand{\A}{\field{A}}

\newcommand{\thetab}{\boldsymbol\theta}
\newcommand{\hthetab}{\boldsymbol{\hat{\theta}}}

\journal{Neurocomputing}

\begin{document}

\begin{frontmatter}



\title{Feature ranking for multi-label classification using Markov Networks}


\author{Pawe{\l} Teisseyre}

\address{Institute of Computer Science, Polish Academy of Sciences \\
              Jana Kazimierza 5 01-248 Warsaw, Poland}
\ead{teisseyrep@ipipan.waw.pl}
\ead[url]{http://www.ipipan.eu/~teisseyrep/}
\begin{abstract}
We propose a simple and efficient method for ranking features in multi-label classification.
The method produces a ranking of features showing their relevance in predicting labels, which in turn allows to choose a final subset of features.
 The procedure is based on Markov Networks and allows to model the dependencies between labels and  features in a direct way. In the first step we build a simple network using only labels and then we test how much adding a single feature affects the initial network. More specifically, in the first step we use the Ising model whereas the second step is based on the score statistic, which allows to test a significance of added features very quickly. The proposed approach does not require transformation of label space, gives interpretable results and allows for attractive visualization of dependency structure.
We give a theoretical justification of the procedure by
discussing some theoretical properties of the Ising model and the score statistic.
We also discuss feature ranking procedure based on fitting Ising model using $l_1$ regularized logistic regressions. Numerical experiments show that the proposed methods outperform the conventional approaches on the considered artificial and real datasets. 

\end{abstract}


\end{frontmatter}


\section{Introduction}
Multi-label classification (MLC) has recently attracted a significant attention, motivated by an increasing number of applications. Examples include
text categorization \citep{SchapireSinger2000, Katakisal2001, Nguyen2005, LozaFurnkranz2008, Rubinetal2012},
image classification \citep{Wang2008, Shottonetal2009, Kumaretal2009},
video classification \citep{Boutelletal2004, Wang2011},
music categorization \citep{Trohidisetal2008},
gene and protein function prediction \citep{Elisseeff2001, Diplarisetal2005, Barutcuoglu2006},
medical diagnosis \citep{Lappenschaaral2005, Abbas2013},
chemical analysis \citep{KawaiTakahashi2009, Mammadovetal2007},
social network mining \citep{TangLiu2009, Peters2010} and
direct marketing \citep{Zhangetal2006}.
More examples can be found in \cite{Gibaja2015}, \cite{TsoumakasandKatakis2007}  and \cite{Dembczynskietal2012}. 
The key problem in multi-label learning is how to utilize label dependencies to improve the classification performance, motivated by which number of multi-label algorithms have been proposed in recent years (see \cite{Madjarov2012} for extensive comparison of several methods). The recent progress in MLC is summarized in \cite{Zhang2013} and \cite{Gibaja2015}.
In MLC, each object of our interest (e.g. text, image, patient, etc.) is described by a vector of $p$ features $\x=(x_1,\ldots,x_p)^{T}$ and a vector of $K$ binary labels $\y=(y_1,\ldots,y_K)^{T}$. The main objective is to build a model (using some training examples) which predicts $\y$ based on $\x$. 


One of the trending challenges in MLC is a dimensionality reduction of the feature space \citep{Gibaja2015}, i.e. reducing the dimensionality of the vector $\x$. Usually only some features affect $\y$.
The issue is very important as 
in practical applications, the dimensionality of feature space can be very large. For example in text categorization a standard approach is to use so-called \textit{bag-of-words model} in which frequencies of occurrence of words in a corpora are taken as features. This method generates thousands of features. Moreover, one can also take into account higher degree n-grams (bigrams, trigrams, etc.) and many other types of features (e.g. stylistic features like averaged word length), which further increases the dimensionality of feature vector.
Elimination of redundant features is essential for the following reasons.
First, it allows to reduce the computational burden of MLC procedures. 
Secondly, it improves a prediction accuracy of MLC methods. 
Fitting many MLC models includes estimation of large number of parameters.
It is well known that fitting models with many spurious features increases the variance of estimators and thus decreases the prediction accuracy of the model (see e.g. chapter 7  in \cite{Hastieetal2009}).    
Finally, feature selection methods are used to discover  dependency structure in data. This allows to understand how features affect the labels, which is particularly important in biological and medical applications. For example, in multi-morbidity (co-occurrence of two or more chronic medical conditions in one person) it is crucial to discover which characteristics of the patient influence the co-occurrence of diseases \citep{Bromurietal2014}.
Moreover, it would be interesting to know which diseases are likely to occur simultaneously given some characteristics of the patient (for example age, gender and previous diseases). We discuss different approaches of dimensionality reduction in MLC in Section \ref{Related work}.

In this paper we focus on Feature Ranking (FR) methods (sometimes also called filters). 
Although the MLC attracted a significant attention in machine learning community, only a few works address the feature ranking problem in multi-label setting. 
Feature ranking (FR) methods are mainly used to assess the individual relevance of available features. More precisely, they allow to order features with respect to their relevance in predicting labels, which in turn allows to remove the least significant features and build a final classification model using the most significant features. 
Although usually in this approach neither the possible redundancy between features nor their joint relevancy is   taken into account, the main advantage is a low computational cost, which allows to compute the importance of thousands of features  relatively fast.
This is crucial in many domains, like text categorization or functional genomics. 
Moreover, in some applications it is important to evaluate the individual relevance of features, not only their joint relevance.
Some authors use FR methods as an initial step to filter out spurious features and then use more sophisticated selection methods on the remaining set of features (see e.g. Sure Independence Screening procedure proposed by \cite{FanLv2008}). 
We also discuss FR method, which incorporates all features simultaneously.

The FR task in multi-label setting is much more challenging than in a single-label case. In traditional classification with only one target variable, FR methods aim to model the dependence between target variable $y$ and a single feature $x_j$ using different variable importance measures. Then the procedure is repeated for all possible features. The most popular measures are: information gain (\cite{Pengetal2005}), the chi-squared statistic and simple statistics based on univariate logistic regression (\cite{Fanetal2009}), among others.     
On the other hand, in MLC feature $x_j$ may affect targets $y_1,\ldots,y_K$ in different ways. First, it may happen that $x_j$ influences only some of labels, while others are independent from $x_j$. More importantly, since in MLC methods dependencies between labels are usually considered, we should verify how $x_j$ affects a given label $y_k$, in a presence of the remaining labels. It may happen that $x_j$ is independent from $y_k$, while $x_j$ becomes dependent on $y_k$, when conditioned on other labels. Finally, feature $x_j$ can influence only the interactions between labels, while the marginal dependencies are not present.
Examples of such situations are provided in Sections \ref{Ising model with constant interaction terms} and \ref{Ising model with feature-dependent interaction terms}.
 A desirable FR method should take into account all the above aspects.

The main limitation of recent FR methods is that they require problem transformation methods: Binary Relevance (BR) or Label Powerset (LP) transformation  for evaluating the relevance of given features. Unfortunately, both  transformations suffer from many serious drawbacks, discussed in more detail in Section \ref{Related work}. To propose a desirable FR method, we make an effort to take into account  the following aspects.
\begin{itemize}
\item The method should not use BR or LP transformation.
\item The method should take into account specificity of multi-label setting, i.e. it should measure the dependence between feature $x_j$ and label $y_k$, given the remaining labels.
\item The method should give interpretable results to see which labels (or interactions between labels) and how are influenced by feature.
\item The computational cost of the procedure should be low.
\end{itemize}
To take into account the above postulates, we propose a novel approach which is based on Markov Networks. Markov Network (see e.g. \cite{Bishop2006}, Section 8) can be represented as a graph, with node set representing random variables (in our case labels and features) and edge set representing dependencies between variables.
Existing edge between two variables means that they are conditionally dependent given the rest of the graph.
The main advantage of Markov Networks is that they allow to model the pairwise dependencies between labels and features in a direct way.
Although, Markov Networks have already been applied in MLC (see e.g. \cite{Chengetal2014} or \cite{Bianetal2012}), they have not been used as a feature ranking method.
Our approach is based on the following idea. We initially build a Markov Network containing only labels, which allows to model the dependencies among the labels. In the second step, we test how much adding a single feature $x_j$ affects the initial network. This allows to test the dependence strength between a given feature $x_j$ and a given label $y_k$, conditioning on the remaining labels. The procedure is repeated for all available features, which yields the final ranking. 
Specifically, in our method we use the Ising model (\cite{Ising1925}, \cite{Lenz1920}) which is a simple example of Markov Network. It turns out that for the Ising model, building an initial network containing only labels can be done relatively simply, especially for moderate number of labels.
Please see Section \ref{Why Ising model} for deeper justification of using the Ising model.
 In a second step we propose to use the score statistic \citep{Rao1948}, which is very computationally efficient in this case. Namely, it is not necessary to refit an initial network when we add feature $x_j$. This allows to test a significance of added features very quickly which is crucial in FR methods. The details of the procedure are given in next Sections. 
Figure \ref{fig001} shows networks corresponding to the most and the least significant features for \texttt{scene} dataset, in which the task is to predict six labels (beach, sunset, field, fall, mountain, and urban). Numbers over edges $u_1,\ldots,u_6$ are the score statistics which reflect the conditional dependences between feature $x_j$ and labels (given the remaining labels). 
The higher the value of the score statistic, the larger is the influence of $x_j$ on the given label, in the presence of remaining labels. The score statistics for a given feature $x_j$ are added together, which gives an importance measure for $x_j$. The final ranking is based on these importances.
We also discuss FR procedure based on fitting Ising model using $l_1$ regularized logistic regressions.

The rest of the article is organized as follows.
In Section \ref{Related work} we discuss the existing related work. 
In Section \ref{Feature importance measures based on the Ising model} we present feature importance measures based on the Ising model. We describe two versions of the Ising model: the first one assumes constant interactions between labels, the second one considers feature-dependent interactions. We discuss some theoretical properties of the score statistic and justify using the Ising model. In addition we discuss a version of the Ising model which incorporates all features simultaneously and describe the estimation procedure based on $l_1$ regularized logistic regressions. 
Section \ref{Feature ranking methods and feature selection methods} contains the formal description of our feature ranking procedures as well as feature selection procedure. We present the results of experiments in Section \ref{Experimental results}. Section \ref{Conclusions} concludes with a summary. The technical proofs are provided in Appendix.

\begin{figure}
\begin{center}$
\begin{array}{cc}
\includegraphics[scale=0.25]{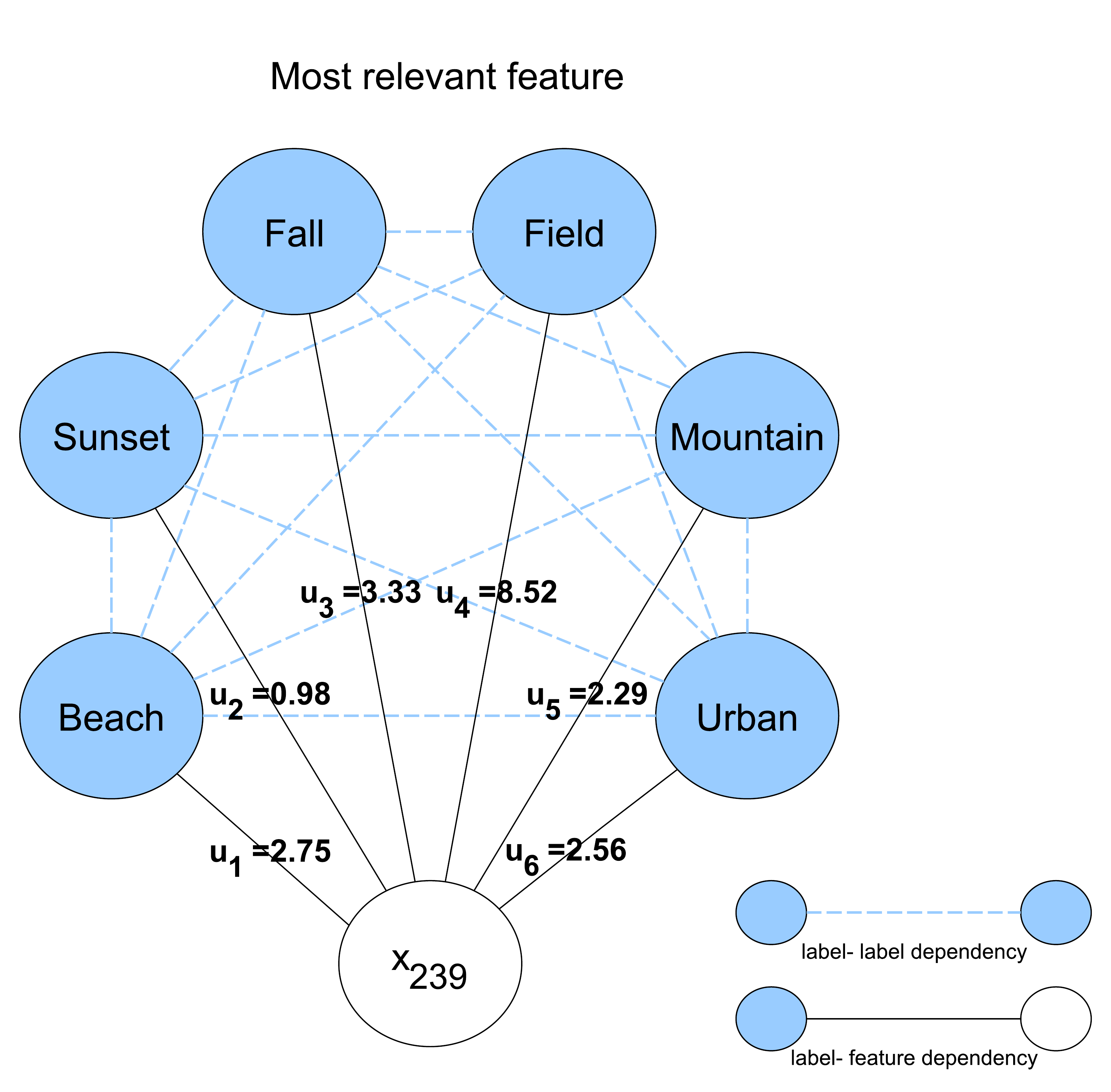} &
\includegraphics[scale=0.25]{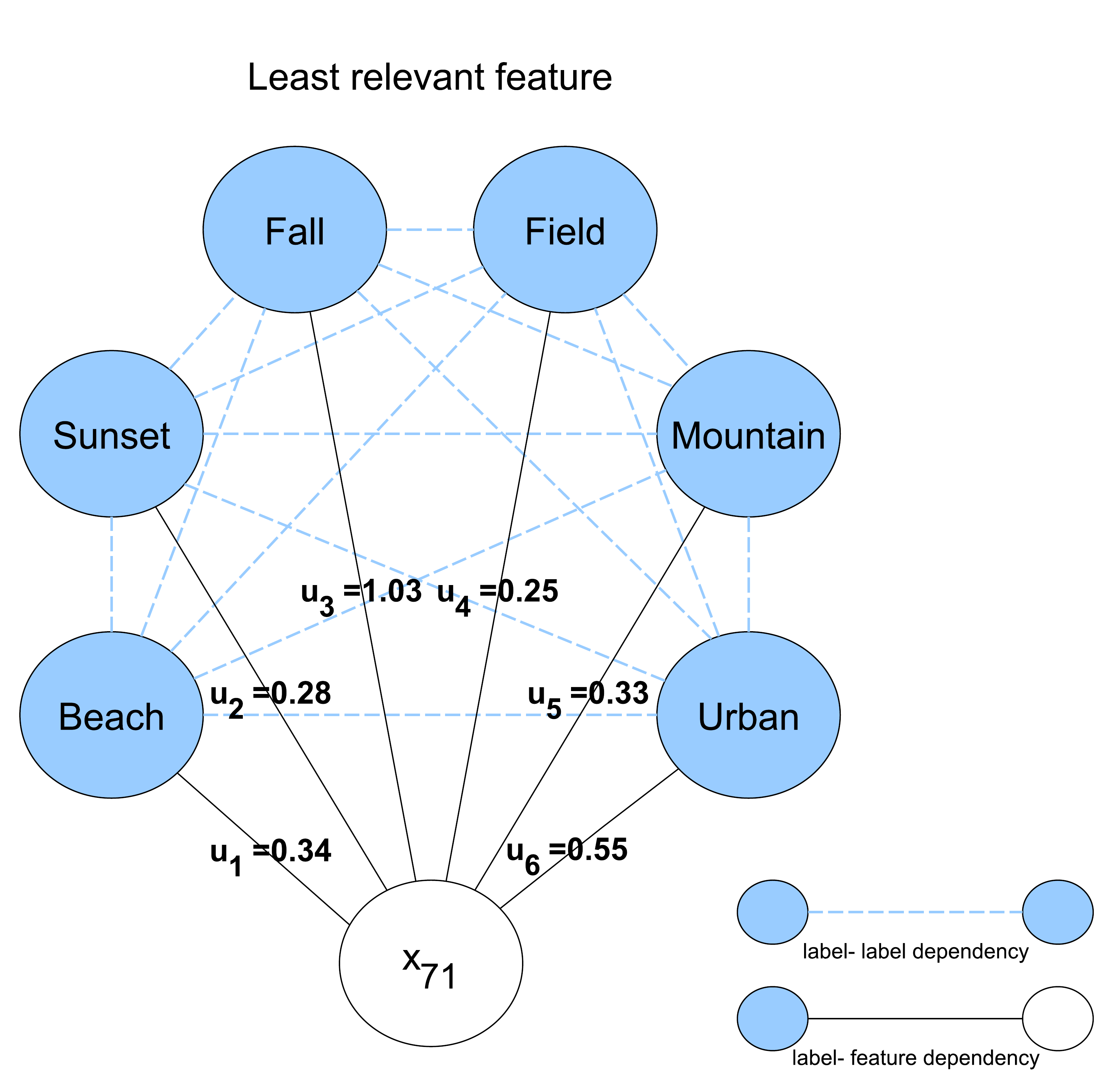} \\
\end{array}$
\end{center}
\caption{Markov networks corresponding to the most ($x_{239}$) and the least ($x_{71}$) significant features for \texttt{scene} dataset. The numbers over edges are scores statistics describing importances of features.}
\label{fig001}
\end{figure}

\section{Related work}
\label{Related work}
Dimensionality reduction of the feature space is one of the current challenges in MLC \citep{Gibaja2015}. There are different approaches to reduce the dimensionality of feature space. The two main groups are: feature selection methods (among which one can distinguish: feature ranking methods, wrappers, embedded methods) and feature transformation methods. 
Feature selection methods aim to identify a small subset of features which influence labels. Feature transformation methods aim to find functions of features, that can replace the original ones.
In this paper we focus on feature ranking methods (FR), which belong to the first group.
FR methods produce a list of features, ordered with respect to their relevance. The final model is built using the most relevant features from the list.

Let us first review the existing FR methods in MLC. The popular approach is to use Binary Relevance (BR) transformation (by considering classification tasks corresponding to separate labels) and to evaluate the relevance of each feature for each of the labels independently (\cite{Trohidisetal2008}, \cite{Chenetal2007}). 
The scores corresponding to different labels are then combined, which yields the global ranking of features.
To evaluate the relevance of features in the tasks, various feature importance measures are used, among which the chi-squared statistic and information gain are the most popular ones (\cite{Spolaor2013}). 
The major drawback of this approach is that possible dependencies between labels are not utilized.
The combinations of BR transformation with the chi-squared statistic and information gain will be referred to as \textit{br chi2} and \textit{br ig}, respectively.

The second popular group of methods is based on LP transformation (\cite{TsoumakasandKatakis2007}) which reduces the multi-label problem to single-label problem with many classes by considering each combination of labels as a distinct meta-class. LP transformation combined with the chi-squared statistic has been used in music classification \citep{Trohidisetal2008}.
This approach requires discretization of features which may lead to loss of some information.
 \cite{DoquireVerleysen2013} proposed to combine LP method and information gain (mutual information), whose estimation is in general a challenging task. They used Kozachenko-Leonenko estimator of entropy \citep{Kozachenko1987}, which is based on nearest neighbours method.
 The approach was also successfully used to assess the relevance of subset of features, not only the individual significance of features, which is a big advantage. The limitation is that the presence of points having the exact same feature values may harm the estimation of entropy based on nearest neighbours. It turns out that the method based on information gain usually outperforms the chi-squared-based approach \citep{DoquireVerleysen2013}. Feature selection based on information gain was also described in \cite{LeeKim2013}.  
The combinations of LP transformation with the chi-squared statistic and information gain will be referred to as \textit{lp chi2} and \textit{lp ig}, respectively.
Although, LP transformation is very simple, it suffers from many serious drawbacks. 
First of all, the number of possible meta-classes can be very large, even larger than the number of observations.
As a result some meta-classes can be represented by a small amount of data and the performance of learning algorithm can be degraded. \cite{Read2008} proposed the Pruned Problem Transformation (PPT) to improve the LP; patterns with too rarely occurring labels are simply removed from the training set by considering label sets with predefined minimum occurrence $\tau$. This modification was also used by \cite{DoquireVerleysen2013}.
The main limitations of this approach are: loss of class information due to removing some observations and the necessity  of choosing the optimal value of $\tau$. Apart from the above drawbacks, in LP-based methods we loose information about dependency structure, i.e. about which labels and how are influenced by a given feature.

Finally, let us also discuss other methods used for dimensionality reduction.
Wrappers allow to assess subsets of features using some criterion function, e.g. prediction error on validation set. To avoid fitting models on all possible subsets, usually some search strategies are used, e.g. forward selection or backward elimination. The main limitation of wrappers is a significant computational cost, due to training large number of classifiers.  
Another important group of methods are so-called embedded feature selection procedures, in which the selection of features is an integral element of the learning process. Examples from this group are: multi-label version of decision trees proposed by \cite{Clare2001} in which the useful features are chosen during building the tree or methods based on $l_1$ regularization \citep{Ravikumaretal2010, Chengetal2014}.

The other important group are feature transformation methods which aim to identify functions of features that can replace the original ones,  e.g. Principal Component Regression \citep{Jolliffe1982} or Partial Least Squares Regression \citep{Martens2001, Wold2001}. Recently Partial Least Squares method has been successfully used in MLC \citep{Liu2015}. Let us also mention about using Canonical Correlation Analysis in multi-label learning \citep{Sunetal2011}. The comprehensive list of feature transformation methods in MLC is given in \cite{Sunetal2014}.
 
\section{Feature importance measures based on the Ising model}
\label{Feature importance measures based on the Ising model}
Before formal description of our method, let us introduce some basic notations for the multi-label learning.
For the convenience of a reader, vectors and matrices are written in bold.
Let $\y=(y_1,\ldots,y_K)^{T}$ be a label vector containing $K$ binary labels and let $\x=(x_1,\ldots,x_p)^{T}$ be a set of $p$ input features.
By $\y_{-k}$ we denote vector $\y$ with $k$-th label removed.
Further, let $\Y$ ($n\times K$) and $\X$ ($n\times p$) be matrices containing instances of $\y$ and $\x$, respectively, in rows. Analogously, let $\Y_{k}$ be $k$-th column of $\Y$ and $\Y_{-k}$ be a matrix $\Y$ with $k$-th column removed. Similarly, let $\X_{j}$ be $j$-th column of $\X$. Finally, the superscript $(i)$ will correspond to $i$-th instance, e.g. $\X^{(i)}$ is $i$-th row of $\X$ and $\X_{j}^{(i)}$ is $i$-th instance of $j$-th column (feature) of $\X$.
The main task in multi-label learning is to build a model based on training data $(\X,\Y)$ which predicts unknown labels for some new objects.
The main goal of FR methods is to evaluate the relevance of features $x_1,\ldots,x_p$ in predicting labels $y_1,\ldots,y_K$, based on data $(\X,\Y)$.
For simplicity, we denote by $y_k\sim x_1,\ldots,x_p$ a classification problem in which $y_k$ is a response (target) variable and $x_1,\ldots,x_p$ are input features.

\subsection{Ising model with constant interaction terms}
\label{Ising model with constant interaction terms}
We start from a simple model in which interactions between labels do not depend on features.
To assess how the individual feature $x_j$ influences the joint distribution of labels we use the Ising model 
\begin{equation} 
\label{ising1}
P(y_1,\ldots,y_K|x_j)=\frac{1}{N(x_j)}\exp\left[\sum_{k=1}^{K}a_kx_jy_k+\sum_{k<l}\beta_{k,l}y_ky_l\right],
\end{equation}
where $a_k, \beta_{k,l}\in R$ are parameters and
\begin{equation}
N(x_j)=\sum_{\y\in\{0,1\}^{K}}\exp\left[\sum_{k=1}^{K}a_kx_jy_k+\sum_{k<l}\beta_{k,l}y_ky_l\right]
\end{equation}
is normalizing constant which ensures that 
the exponential functions sum up to $1$. Note that the normalizing term depends on $x_j$ but does not depend on $\y$. It is assumed that $\beta_{k,l}=\beta_{l,k}$. 
Parameters $a_k$ describe the individual contribution of the labels, whereas $\beta_{k,l}$ correspond to interactions between labels. 
Note that our model (\ref{ising1}) is identical to CORRLog model used in \cite{Bianetal2012}.
Number of authors consider unconditional version of (\ref{ising1}) to model $P(y_1,\ldots,y_n)$ (e.g. \cite{Ravikumaretal2010}). In statistical literature, the unconditional version of (\ref{ising1}) is referred to as auto-logistic model \citep{Besag1972, ZalewskaNiemiro2010}.
Model (\ref{ising1}) describes a simple Markov Network (or more specifically Conditional Random Field, \cite{Laffertyetal2001}) in which vertices correspond to labels and a given feature whereas edges correspond to dependencies. Labels $y_k$ and $y_l$ are conditionally independent given $x_j$ and all other labels if and only if $\beta_{k,l}=0$ (no edge between $y_k$ and $y_l$). The advantage of the above model is that it indicates which labels are influenced by feature $x_j$.
The feature $x_j$ is not relevant when $a_1=\ldots=a_K=0$ (no edges between $x_j$ and the rest of the graph). A natural way to assess the relevance of $x_j$ would be to estimate parameters in model (\ref{ising1}) (using e.g. maximum likelihood approach) and then to perform some statistical test to verify whether $a_k\neq 0$. However it is difficult to estimate unknown parameters in (\ref{ising1}) directly by maximizing the joint conditional log-likelihood since the probability in (\ref{ising1}) includes the normalizing term, which requires summation of $2^{K}$ terms for each data point and makes it intractable for direct maximization. Instead we use simple procedure via node-wise regressions 
suggested in \cite{Ravikumaretal2010}. First it is easy to verify (see \ref{Appendix A} for the proof) that
\begin{equation}	
\label{logodds}  
\log\left[\frac{P(y_k=1|x_j,\y_{-k})}{P(y_k=0|x_j,\y_{-k})}\right]=
\sum_{l:l\neq k}\beta_{k,l}y_l+a_kx_j.
\end{equation} 
It follows from (\ref{logodds}) that in order to estimate parameter vector
\begin{equation*}
\thetab_k=(\beta_{k,1},\ldots,\beta_{k,k-1},\beta_{k,k+1},\ldots,\beta_{k,K},a_k)^{T}\in R^{K},
\end{equation*}
it suffices to fit logistic model $y_k\sim \y_{-k},x_j$ in which $y_k$ is a response variable, whereas labels $y_1,\ldots,y_{k-1},y_{k+1},\ldots,y_{K}$  and feature $x_j$ are the explanatory variables.
The crucial in the above idea, is that the normalizing term $N(x_j)$ is eliminated.
 Now to assess the relevance of feature $x_j$, we propose to use the score statistic \citep{Rao1948} to compare logistic models $y_k\sim \y_{-k}$ and $y_k\sim \y_{-k},x_j$.
Let 
\begin{equation*}
\hthetab_k=(\hat{\beta}_{k,1},\ldots,\hat{\beta}_{k,k-1},\hat{\beta}_{k,k+1},\ldots,\hat{\beta}_{k,K},0)^{T}\in R^{K}
\end{equation*}
 be the maximum likelihood estimator in the smaller model $y_k\sim \y_{-k}$. We augment it to $K$- dimensional vector by setting the last coordinate to $0$. 
 Define $n\times K$ matrix $\Z=(\Y_{-k},\X_{j})$.  
In our setting, the score statistic is defined as
\begin{equation}
\label{score}
u_k(x_j):=|s^2(\hthetab_k)/v(\hthetab_k)|,
\end{equation} 
where
\begin{equation*}
s(\hthetab_k):=\sum_{i=1}^{n}\X_{j}^{(i)}(\Y_{k}^{(i)}-p^{(i)}(\hthetab_k)),
\end{equation*}
\begin{equation*}
p^{(i)}(\hthetab_k)=\frac{\exp(\hthetab_k^{T}\Z^{(i)})}{1+\exp(\hthetab_k^{T}\Z^{(i)})},
\end{equation*}
\begin{equation*}
v(\hthetab_k)=D(\hthetab_k)-\C(\hthetab_k)\A^{-1}(\hthetab_k)\B(\hthetab_k)
\end{equation*}
where
\begin{equation*}
\A(\hthetab_k)=\Y_{-k}^{T}\W(\hthetab_k)\Y_{-k},
\end{equation*}
\begin{equation*}
\B(\hthetab_k)=\Y_{-k}^{T}\W(\hthetab_k)\X_j,
\end{equation*}
\begin{equation*}
\C(\hthetab_k)=\X_{j}^{T}\W(\hthetab_k)\Y_{-k},
\end{equation*}
\begin{equation*}
D(\hthetab_k)=\X_{j}^{T}\W(\hthetab_k)\X_{j}
\end{equation*}
and
$\W(\hthetab_k)$ is $n\times n$ diagonal matrix with $p^{(i)}(\hthetab_k)(1-p^{(i)}(\hthetab_k))$ on diagonal.
Observe that $s(\hthetab_k)$ measures the correlation between added feature and residuals obtained from the smaller model.
The main advantage of using the score statistic is that $\hthetab_k, \W(\hthetab_k)$ and $\A^{-1}(\hthetab_k)$ do not involve $\X_j$ and thus these terms need to be calculated only once. 
Computing the remaining terms: $\B(\hthetab_k)$, $\C(\hthetab_k)$, $D(\hthetab_k)$ and $s(\hthetab_k)$ can be done very quickly and stably, even for thousands of features $x_j$. 
So computation of the score statistics requires fitting only the smaller model $y_k\sim\y_{-k}$.
This is not the case for other popular statistics like the Wald statistic or the Likelihood Ratio statistic \citep{Fahrmeir1987} which involve fitting both $y_k\sim \y_{-k}$ and $y_k\sim \y_{-k},x_j$ models.
In Section \ref{Properties of the score statistic} we prove that, for relevant feature $x_j$ ($a_k\neq 0$), the score statistic tends to infinity when sample size increases and moreover we show that the lower bound of the score statistic is an increasing function of $|a_k|$.

Observe that the larger the value of $u_k(x_j)$, the more important is a feature $x_j$ in model $y_k\sim \y_{-k},x_j$. We check the usefulness of $x_j$ for predicting $k$-th label, when all remaining labels are present in the model. In other words, we test how much adding a feature $x_j$ to labels $\y_{-k}$ improves the prediction of $y_{k}$. 
Consider the following toy examples with one feature $x_1$ and two labels $y_1,y_2$.
The example shows that adding feature $x_1$ to $y_2$ may improve prediction of $y_1$. Consider two binary labels $y_1,y_2$, such that $P(y_2=1)=0.5$ and binary feature $x_1$, such that $P(x_1=1)=0.5$ and assume that $x_1$ is independent from $y_2$ and $y_1=I(y_2+x_1>0)$ (where $I$ is indicator function). It is seen that $y_1$ can be predicted by $y_2$ with maximal accuracy $75\%$ and similarly $y_1$ can be predicted by $x_1$ with accuracy $75\%$. On the other hand when $y_1$ is explained by both $y_2$ and $x_1$, the accuracy is $100\%$. 

\subsection{Properties of the score statistic}
\label{Properties of the score statistic}
In this section we study some theoretical properties of the score statistic (\ref{score}). Recall that the score statistic is used to test the significance of feature $x_j$ in model $y_k\sim\y_{-k},x_j$. 
The score statistic is a classical measure, proposed more than 60 years ago \citep{Rao1948}, however recently it has attracted again a significant attention in high-dimensional problems, mainly because of its low computational cost and good performance. For example, the score statistic has been successfully used for feature ranking in analysing Genome Wide Association Studies \citep{HeLin2011}. 

It is well known fact that when $x_j$ is not significant, i.e. $a_k=0$, then, under some regularity conditions, the score statistic $u_{k}(x_j)$ is approximately distributed as chi-squared with $1$ degree of freedom, for large sample size $n$ (see e.g. \cite{Fahrmeir1987}). Thus in the following we will focus on the performance of $u_{k}(x_j)$, when $x_j$ is significant, i.e. $a_k\neq 0$. Best of our knowledge, the properties of the score statistic under this setting has not yet been discussed.

So assume that $a_k\neq 0$ and we fit the smaller model $y_k\sim \y_{-k}$ from which we have an estimator 
\begin{equation*}
\hthetab_k=(\hat{\beta}_{k,1},\ldots,\hat{\beta}_{k,k-1},\hat{\beta}_{k,k+1},\ldots,\hat{\beta}_{k,K},0)^{T},
\end{equation*}
with the coordinate corresponding to $x_j$ set to $0$.
Recall that $\Z=(\Y_{-k},\X_{j})$.
Let $\lambda_{\min}(\A)$ be the minimal eigenvalue of matrix $\A$. Define $L:=\max_{i,j}|\X^{(i)}_j|$ (to simplify a proof we assume $L>1$), 
$\Lambda_{\min}:=\lambda_{\min}(\Z^{T}\Z/n)$ and let $-G \leq a_k\leq G$. Constant G determines the range of unknown parameter $a_k$, corresponding to variable $x_j$. This constant is introduced to facilitate the proof of Theorem \ref{Theorem 1}.
\begin{Theorem}
\label{Theorem 1}
The following inequality holds
\begin{equation*}
P\left[u_{k}(x_j)\geq \frac{C n a_k^2}{H^4}\Big|\Z\right]\geq 1-K\exp\left[-\frac{Cn(K+L^2)a_k^{2}}{2H^2}\right],
\end{equation*}
where $C=\left(\frac{\Lambda_{\min}v}{2e^{3}L(K+L)^{3/2}}\right)^{2}$, $H=\max(1,G)$ and 
$v=\min_{i}p^{(i)}(\thetab_k)(1-p^{(i)}(\thetab_k))$.
\end{Theorem}
The proof of the above result is provided in \ref{Proof of Theorem1}.
Let us discuss the meaning of the above result as well as effects of different constants. 
It follows from the above Theorem that $u_{k}(x_j)\to\infty$, with probability tending to $1$, as $n\to\infty$, which is a desired result as $u_{k}(x_j)$ should take large values when $x_j$ is significant.
Moreover, it is seen that the lower bound $C n a_k^2/H^4$ is an increasing function of $|a_k|$ and decreasing function of $K$, which 
is concordant with intuition.
Indeed, the larger the value of $a_k$, the more significant is the feature $x_j$. It is very useful property as it allows to assess the significance of the feature $x_j$, without estimating the corresponding unknown coefficient $a_k$. For large value of $K$, it is more difficult to test the significance of the feature $x_j$. 
Constant $v$ is a minimal (where minimum is taken over all training examples) variance of $k$-th label, conditioned on the remaining labels and feature $x_j$. Small value of $v$ indicates that the classes, corresponding to $k$-th label, are almost separable and thus the logistic model may fail.
Very small value of $\Lambda_{\min}$ indicates that columns of matrix $\Z$ are almost linearly dependent, which may harm the fitting of logistic model. We show in Lemma \ref{Lemma4} that $\Lambda_{\min}>0$ ensures that the likelihood function corresponding to larger model is strictly concave. In addition observe that the lower bound $C n a_k^2/H^4$ is a decreasing function of $v$ and $\Lambda_{\min}$, which is again intuitive: the more difficult the problem, the more challenging is identification of a significant variable. Finally, constants $H$ and $G$ are introduced for technical reasons, to facilitate the proof.

The following example illustrates the above theoretical result on artificial data.
Consider one feature $x_1$ and ten labels $y_1,\ldots,y_{10}$.
 We generate labels $y_2,\ldots,y_{10}$ independently, from binomial distribution with success probability $0.5$. Then we generate $y_1$, from (\ref{logodds}), with $\beta_{k,l}=0.1$, and $x_1$ drawn from standard Gaussian distribution.
The simulations are repeated $50$ times.
Figure \ref{fig00} (a) shows smoothed histograms of the score statistics when $x_1$ is relevant ($a_1=1$) and irrelevant ($a_1=0$). In the latter case, the values of the score statistics remain close to zero. Figure \ref{fig00} (b) shows the score statistics w.r.t. increasing value of $a_1$ (coefficient corresponding to $x_1$), for different sample sizes. It is clearly seen that the larger the value of the coefficient, the larger the value of the score statistic.

\begin{figure}
\begin{center}$
\begin{array}{cc}
\includegraphics[scale=0.45]{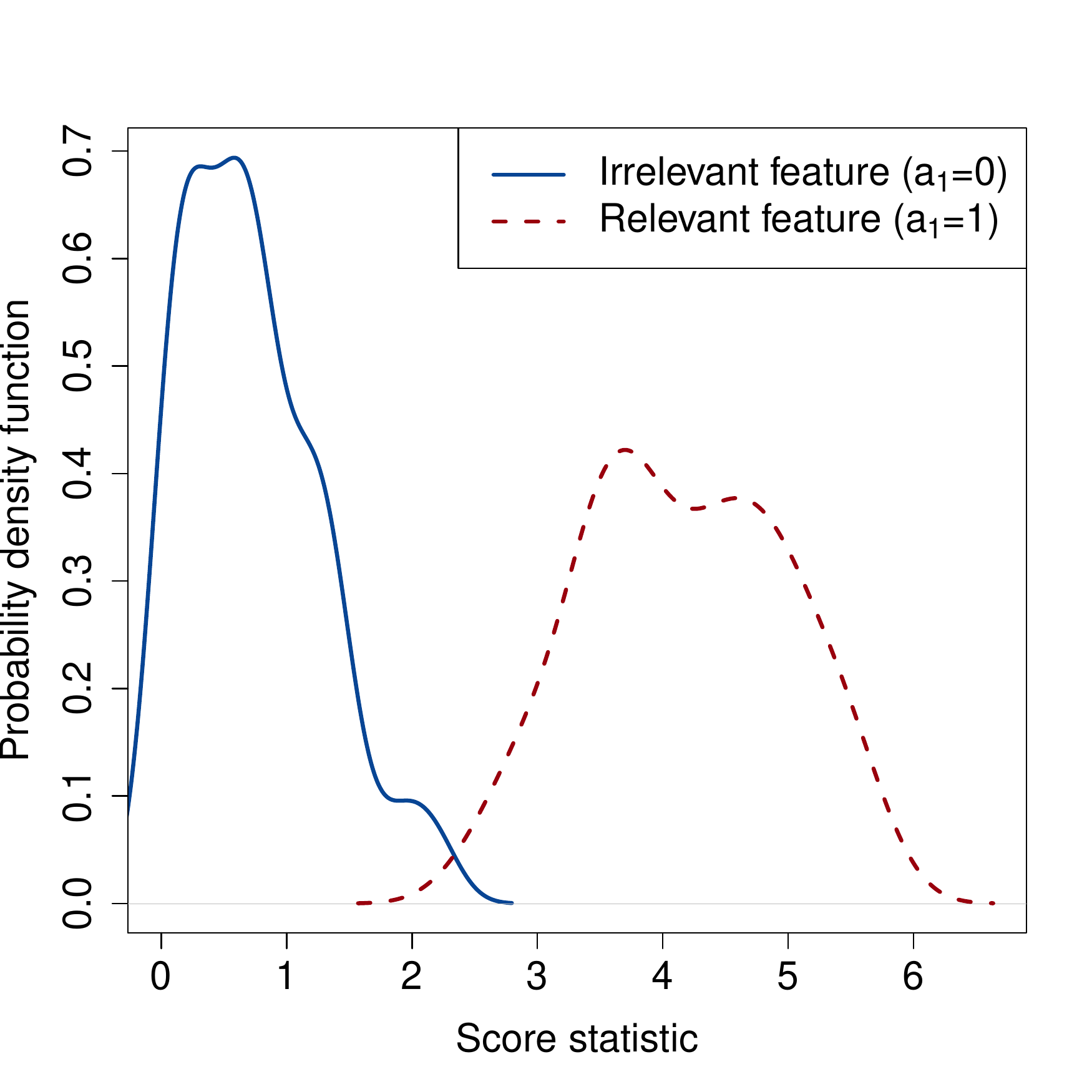} &
\includegraphics[scale=0.45]{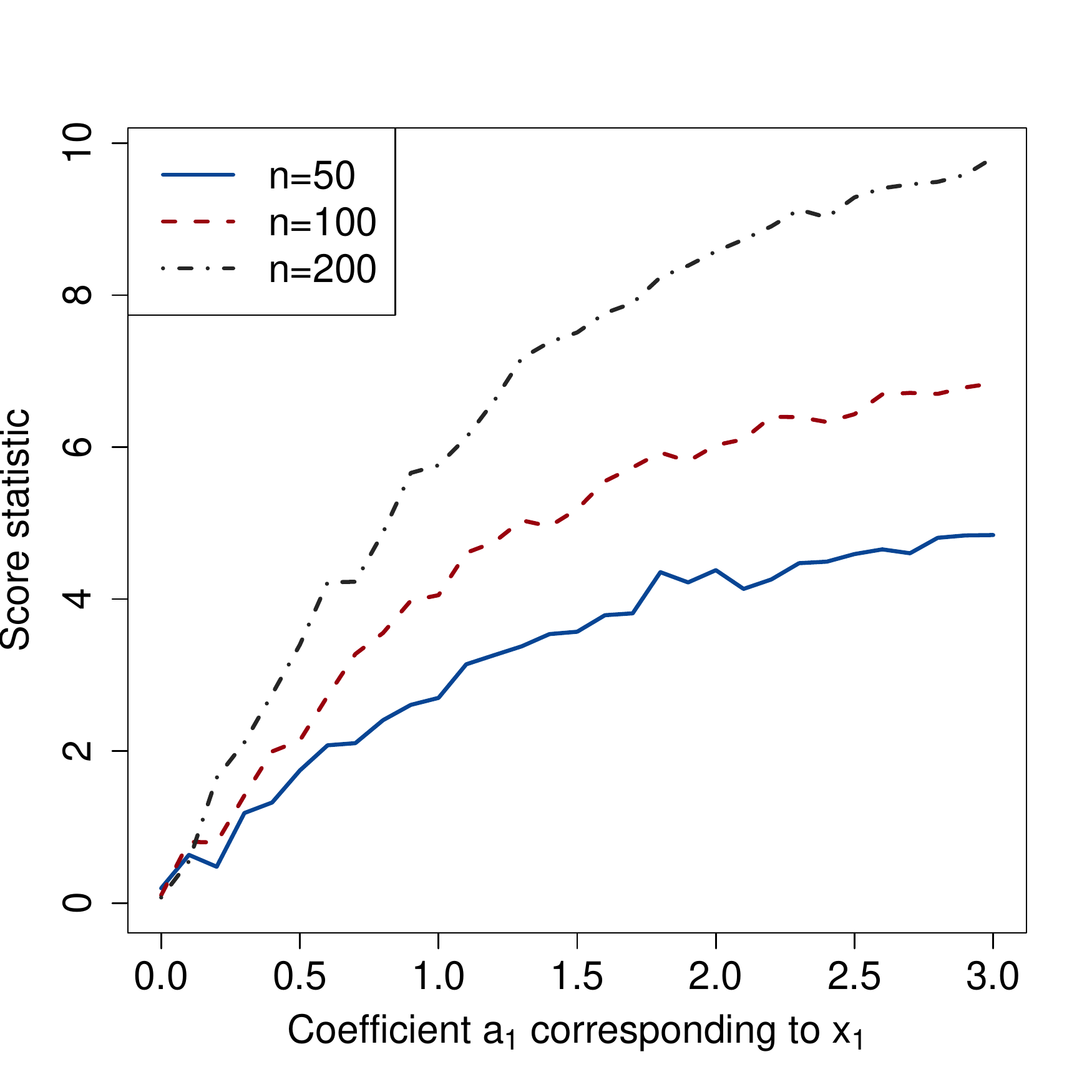} \\
\textrm{(a)} & \textrm{(b)} \\
\end{array}$
\end{center}
\caption{Performance of the score statistics for artificial data.}
\label{fig00}
\end{figure}

\subsection{Ising model with feature-dependent interaction terms}
\label{Ising model with feature-dependent interaction terms}
In real applications, it is often the case that interactions between labels depend on features. In order to model this situation, we expand model (\ref{ising1}) as
\begin{equation} 
\label{ising2}
P(y_1,\ldots,y_K|x_j)=\frac{1}{N(x_j)}\exp\left[\sum_{k=1}^{K}a_kx_jy_k+\sum_{k<l}(\beta_{k,l}+b_{k,l}x_j)y_ky_l\right],
\end{equation}
where $b_{k,l}$ describes how strong feature $x_j$ influences the interactions between $y_k$ and $y_l$.
The price for considering feature-dependent interactions is larger number of parameters.
Analogous reasoning as in (\ref{logodds}) leads to
\begin{equation}	
\label{logodds2}  
\log\left[\frac{P(y_k=1|x_j,\y_{-k})}{P(y_k=0|x_j,\y_{-k})}\right]=
\sum_{l:l\neq k}\beta_{k,l}y_j+\sum_{l:l\neq k}b_{k,l}x_jy_l+a_kx_j.
\end{equation} 
It follows from (\ref{logodds2}) that in order to estimate $\beta_{k,l},b_{k,l},a_k$, it suffices to fit logistic model $y_k\sim \y_{-k},x_j\y_{-k},x_j$ in which $y_k$ is a response variable, whereas $\y_{-k}$, $x_j\y_{-k}$ and $x_j$ are the explanatory variables.
Now to assess the relevance of feature $x_j$, we compare models $y_k\sim \y_{-k}$ and $y_k\sim \y_{-k},x_j\y_{-k},x_j$. 
Define vector $\m=(x_j\y_{-k},x_j)^{T}$ and let $M$ be $n\times K$ matrix containing realizations of $\m$ in rows.
Let $\hthetab_k$ be an estimator based on smaller model $y_k\sim\y_{-k}$.
Multivariate version of the score statistic is defined as 
\begin{equation}
\label{score1}
U_k(x_j):=|\Sb^{T}(\hthetab_k)\V^{-1}(\hthetab_k)\Sb(\hthetab_k)|,
\end{equation} 
where
\begin{equation*}
\Sb(\hthetab_k):=\M^{T}(\Y_{k}-\p(\hthetab_k)),
\end{equation*}
\begin{equation*}
\p(\hthetab_k)=(p^{(1)}(\hthetab_k),\ldots,p^{(n)}(\hthetab_k))^{T},
\end{equation*}
\begin{equation*}
\V(\hthetab_k)=\D(\hthetab_k)-\C(\hthetab_k)\A^{-1}(\hthetab_k)\B(\hthetab_k)
\end{equation*}
where
\begin{equation*}
\A(\hthetab_k)=\Y_{-k}^{T}\W(\hthetab_k)\Y_{-k},
\end{equation*}
\begin{equation*}
\B(\hthetab_k)=\Y_{-k}^{T}\W(\hthetab_k)\M,
\end{equation*}
\begin{equation*}
\C(\hthetab_k)=\M^{T}\W(\hthetab_k)\Y_{-k},
\end{equation*}
\begin{equation*}
\D(\hthetab_k)=\M^{T}\W(\hthetab_k)\M.
\end{equation*}

\subsection{Ising model and $l_1$ regularization}
\label{Ising model and $l_1$ regularization}
Model (\ref{ising1}) allows to assess how the single feature $x_j$ influences the joint distribution of labels. To investigate how the whole vector of features $\x\in R^{p}$ influences the joint distribution of labels we propose to use the following model
\begin{equation} 
\label{ising3a}
P(y_1,\ldots,y_K|\x)=\frac{1}{N(\x)}\exp\left[\sum_{k=1}^{K}\ba_k^{T}\x y_k+\sum_{k<j}\beta_{k,j}y_ky_j\right],
\end{equation}
where $\ba_k=(a_{k,1},\ldots,a_{k,p})^{T}$ is a $p$-dimensional parameter vector and
\begin{equation}
N(\x)=\sum_{\y\in\{0,1\}^{K}}\exp\left[\sum_{k=1}^{K}\ba_k^{T}\x y_k+\sum_{k<j}\beta_{k,j}y_ky_j\right]
\end{equation}
is normalizing constant.
Note that model (\ref{ising3a}) incorporates all features $\x=(x_1,\ldots,x_p)^{T}$ simultaneously, not only a single feature $x_j$ as in (\ref{ising1}) and (\ref{ising2}). 
Similar reasoning as in (\ref{logodds}) leads to
\begin{equation}	
\label{logodds}  
\log\left[\frac{P(y_k=1|\x,\y_{-k})}{P(y_k=0|\x,\y_{-k})}\right]=
\sum_{l:l\neq k}\beta_{k,l}y_l+\ba_k^{T}\x,
\end{equation} 
which as in the case of (\ref{ising1}) and (\ref{ising2}) suggests that the unknown parameters can be estimated using logistic regression. Since the dimension of $\x$ can be large, we use $l_1$ regularization to estimate parameter vector
\begin{equation*}
\thetab_k=(\beta_{k,1},\ldots,\beta_{k,k-1},\beta_{k,k+1},\ldots,\beta_{k,K},\ba_k)^{T}\in R^{K+p-1},
\end{equation*}
where $\ba_k=(a_{k,1},\ldots,a_{k,p})^{T}$. So the estimate vector is obtained as
\begin{equation}
\label{l1reg}
\hat{\thetab}_k=(\hat{\beta}_{k,1},\ldots,\hat{\beta}_{k,k-1},\hat{\beta}_{k,k+1},\ldots,\hat{\beta}_{k,K},\hat{\ba}_k)^{T}=\arg\min_{\thetab_k\in R^{p+K-1}}\{-l(\thetab_k)+\lambda||\thetab_k||_1\},
\end{equation}
where $l_k(\cdot)$ is a logistic log-likelihood function corresponding to (\ref{l1reg}), $\lambda>0$ is a parameter and $||\cdot||_1$ is $l_1$ norm. The above procedure was used in \cite{Ravikumaretal2010}, \cite{Bianetal2012} and \cite{Chengetal2014}. The advantage of this approach is that we assess the joint relevance of all features, not only the individual relevance of single feature $x_j$, as in the case of models (\ref{ising1}) and (\ref{ising2}). On the other hand, this approach is much more computationally demanding than fitting models (\ref{ising1}) and (\ref{ising2}). 
We will show in Section \ref{Feature ranking methods and feature selection methods} how to use the above procedure to construct the ranking of features.
\subsection{Why Ising model?}
\label{Why Ising model}
The first advantage of the Ising model is that it precisely describes the dependence structure between labels and a given feature $x$, i.e. it indicates which labels and interactions of labels are influenced by feature $x$.
Secondly, it follows from (\ref{logodds}) and (\ref{logodds2}) that fitting the model can be done relatively simply.
Finally, it turns out that the Ising model is a maximum entropy model, i.e. it maximizes the entropy under some constraints on the expectations of labels and interactions between labels. The details are given below. Assume that we would like to find a conditional distribution of labels $g(\y|x)$, which maximizes the entropy $H_{g}(\y|x)=-\sum_{\y}g(\y|x)\log(g(\y|x))$ under the standard constraints $g(\y|x)\geq 0$, $\sum_{\y}g(\y|x)=1$ and two additional constraints:
\begin{equation}
\label{constr1}
\sum_{\y}g(\y|x)y_k=A_{k}(x),\quad k=1,\ldots,K,
\end{equation}
\begin{equation}
\label{constr2}
\sum_{\y}g(\y|x)y_ky_l=B_{k,l}(x), \quad k<l,
\end{equation}
where $A_{k}(x)$ and $B_{k,l}(x)$ are fixed terms dependent on $x$. The above constraints are very natural in multi-label setting and they simply mean that the expectations of labels as well as interactions between labels (with respect to $g(\y|x)$) are fixed. 
The following fact is proved in \ref{Appendix B}.
\begin{Proposition}
\label{Proposition 1}
Let $g(\y|x)$ be any probability function satisfying (\ref{constr1}), (\ref{constr2}) and let $p(\y|x)$ be probability of the form (\ref{ising2}) also satisfying constraints (\ref{constr1}), (\ref{constr2}). Then $H_{g}(\y|x)\leq H_{p}(\y|x)$. 
\end{Proposition}
Thus, according to the above Proposition and the principle of maximum entropy \citep{Jaynes1957}, the Ising distribution is the most adequate one in a situation when constraints (\ref{constr1}), (\ref{constr2}) are taken into account.
\section{Feature ranking methods and feature selection methods}
In this Section, we show how to use the Ising model described in previous Section, to construct rankings of features. The first approach is based on the Ising model with constant interaction terms and the score statistic. The second approach is based on the Ising model with feature-dependent interaction terms and the score statistic. The third approach, computationally most expensive, is based on fitting $l_1$ regularized logistic regressions.  
\label{Feature ranking methods and feature selection methods}
\subsection{Feature ranking based on the Ising model with constant interaction terms}
We use the Ising model with constant interaction terms and the score statistic to construct the first feature importance measure.
As a feature importance measure for feature $x_j$ we take $imp(x_j):=\sum_{k=1}^{K}u_k(x_j)$.
Recall that $u_k(x_j)$ is non-negative.
Note that, the more important is a feature $x_j$ for predicting consecutive labels (conditioning on the remaining labels), the greater is the measure $imp(x_j)$.  
In addition, the more labels are influenced by $x_j$, the greater is the measure $imp(x_j)$.

Based on the above feature importance measure we propose the following FR procedure which consists of two steps. We initially fit the unconditioned Ising model using only labels, which requires fitting $K$ logistic models, with $K-1$ input features, each. The first step is the price for avoiding LP transformation.
Since the models are fitted independently, the loop can be computed in parallel.
In the second step we assess whether adding input features improve the fitting of the model from the first step. The second step is very efficient and allows to assess the importance of thousands of features quickly.
Thus the method is tailored to the case of large number of features and moderate number of labels.
Figure \ref{fig0} shows networks corresponding to these two steps in the case of three labels.
The whole procedure is described by Algorithm \ref{alg1}.  In simulation experiments we will refer to this method as \textit{ising+score}.
\begin{algorithm}[ht!]
\caption{Feature ranking based on the Ising Model with constant interaction terms (\textit{ising+score})}

\SetKwInOut{Input}{Input}
\SetKwInOut{Initialize}{Initialize}
\SetKwInOut{Output}{Output}

\KwData{$\X (n\times p)$, $\Y (n\times K)$}

  $\#$1 step: fitting the unconditioned Ising model:

  \For{k $\leftarrow$ $1$ \KwTo $K$}{
  
  Fit logistic model $y_{k}\sim \y_{-k}$;
	
  Obtain terms not involving $\X$: $\hthetab_k$, $\W(\hthetab_k)$ and $\A^{-1}(\hthetab_k)$;	
 }

 $\#$2 step: feature ranking:
 
  \For{j $\leftarrow$ $1$ \KwTo $p$}{
    \For{k $\leftarrow$ $1$ \KwTo $K$}{
	  $\#$Low computational cost:    
    
	  Calculate terms involving $\X_j$: $\B(\hthetab_k)$, $\C(\hthetab_k)$, and $D(\hthetab_k)$;	    
    
      Compute $u_k(x_j)$ using $\hthetab_k$, $\W(\hthetab_k)$, $\A^{-1}(\hthetab_k)$ and $\B(\hthetab_k)$, $  \C(\hthetab_k)$, $D(\hthetab_k)$;
      
    }
   Compute feature importance measure $imp(x_j)=\sum_{k=1}^{K}u_k(x_j)$;
 }
 
 Sort features with decreasing order of $imp$: $imp(x_{j_1})\geq\ldots\geq imp(x_{j_p})$
 
\KwOut{Ordered list of features $x_{j_1},\ldots,x_{j_p}$}
\label{alg1}
\end{algorithm}
\subsection{Feature ranking based on the Ising model with feature-dependent interaction terms}
We use the Ising model with feature-dependent  interaction terms and the score statistic to construct the second feature importance measure.
As a feature importance measure for feature $x_j$ we can take $imp(x_j):=\sum_{k=1}^{K}U_k(x_j)$. Alternatively, one can verify the individual contribution of each interaction by taking 
$imp(x_j):=\sum_{k=1}^{K}[u_k(x_j)+\sum_{s:s\neq k}u_k(x_jy_s)]$. Although the advantage of former version is that we assess the joint contributions of all interactions, it is computationally more demanding and less stable due to inversion of $\V(\hthetab_k)$. Since the performances of these two versions were similar, in next sections we present the results for the second version.
The whole procedure is described by Algorithm \ref{alg2}.
In simulation experiments we will refer to this method as \textit{ising inter+score}.
The following toy example, with two labels $y_1, y_2$ and one feature $x_1$, shows that method based on the Ising model with constant interactions may fail, whereas the improved version with feature-dependent interactions will succeed. Consider XOR problem in which $x_1=0$ for $(y_1,y_2)=(1,1)$ or $(y_1,y_2)=(0,0)$ and $x_1=1$ for $(y_1,y_2)=(1,0)$ or $(y_1,y_2)=(0,1)$. In this case feature $x_1$ should be recognized as significant as it partly determines the combination of labels although it is independent from both labels. Adding $x_1$ to logistic model $y_1\sim y_2$ does not improve the model fitting and will result in the score statistic close to $0$. On the other hand, adding both $x_1$ and $x_1y_2$ to model $y_1\sim y_2$ improves the model fitting significantly and thus will result in large value of the score statistic. 
\begin{algorithm}[ht!]
\caption{Feature ranking based on the Ising Model with feature-dependent interaction terms (\textit{ising inter+score})}

\SetKwInOut{Input}{Input}
\SetKwInOut{Initialize}{Initialize}
\SetKwInOut{Output}{Output}

\KwData{$\X (n\times p)$, $\Y (n\times K)$}

  $\#$1 step: fitting the unconditioned Ising model:

  \For{k $\leftarrow$ $1$ \KwTo $K$}{
  
  Fit logistic model $y_{k}\sim \y_{-k}$;
	
  Obtain terms not involving $\X$: $\hthetab_k$, $\W(\hthetab_k)$ and $\A^{-1}(\hthetab_k)$;	
 }

 $\#$2 step: feature ranking:
 
  \For{j $\leftarrow$ $1$ \KwTo $p$}{
    \For{k $\leftarrow$ $1$ \KwTo $K$}{
	     
	  Calculate terms involving $\X_j$: $\B(\hthetab_k)$, $\C(\hthetab_k)$, and $D(\hthetab_k)$;	    
    
      Compute $u_k(x_j)$ and $u_k(x_jy_1),\ldots,u_k(x_jy_{k-1}),u_k(x_jy_{k+1}),\ldots,u_k(x_jy_{K})$ using $\hthetab_k$, $\W(\hthetab_k)$, $\A^{-1}(\hthetab_k)$ and $\B(\hthetab_k)$, $  \C(\hthetab_k)$, $D(\hthetab_k)$;
      
    }
   Compute feature importance measure $imp(x_j):=\sum_{k=1}^{K}[u_k(x_j)+\sum_{s:s\neq k}u_k(x_jy_s)]$;
 }
 
 Sort features with decreasing order of $imp$: $imp(x_{j_1})\geq\ldots\geq imp(x_{j_p})$
 
\KwOut{Ordered list of features $x_{j_1},\ldots,x_{j_p}$}
\label{alg2}
\end{algorithm}
\subsection{Feature ranking based on the Ising model and $l_1$ regularization}
The third feature importance measure is based on the Ising model and $l_1$ regularization. 
Let $\hat{\ba}_k=(\hat{a}_{k,1},\ldots,\hat{a}_{k,p})^{T}$ be the estimate vector which optimizes function (\ref{l1reg}). Observe that using $l_1$ regularization encourages sparsity, i.e. some coefficients $\hat{a}_{k,j}$ will be exactly zero. 
We define the feature importance measure as $imp(x_j):=\sum_{k=1}^{K}|\hat{a}_{k,j}|$, i.e. we aggregate the coefficients describing the influence of feature $x_j$ on labels, in a presence of the remaining labels and remaining features. 
Note that $\hat{a}_{k,j}$ depends on parameter $\lambda$ in (\ref{l1reg}); in experiments we take $\lambda=0.0001\lambda_{\max}$, where $\lambda_{\max}$ is a value of $\lambda$ for which all coordinates of $\hat{\thetab}_k$ are exactly $0$. 
The whole procedure is described by Algorithm \ref{alg3}.
In simulation experiments we will refer to this method as \textit{ising+l1}.
\begin{algorithm}[ht!]
\caption{Feature ranking based on the Ising Model and $l_1$ regularization (\textit{ising+l1})}

\SetKwInOut{Input}{Input}
\SetKwInOut{Initialize}{Initialize}
\SetKwInOut{Output}{Output}

\KwData{$\X (n\times p)$, $\Y (n\times K)$}

  $\#$1 step: Fitting the Ising model using $l_1$ regularized logistic regressions

  \For{k $\leftarrow$ $1$ \KwTo $K$}{
  
  Compute $\hat{\thetab}_k=(\hat{\beta}_{k,1},\ldots,\hat{\beta}_{k,k-1},\hat{\beta}_{k,k+1},\ldots,\hat{\beta}_{k,K},\hat{\ba}_k)^{T}=\arg\min_{\thetab_k\in R^{p+K-1}}\{-l_k(\thetab_k)+\lambda||\thetab_k||_1\}$,
  
  where $\hat{\ba}_k=(\hat{a}_{k,1},\ldots,\hat{a}_{k,p})^{T}\in R^{p}$;	
 }

 $\#$2 step: feature ranking:
 
  \For{j $\leftarrow$ $1$ \KwTo $p$}{
    
   Compute feature importance measure $imp(x_j):=\sum_{k=1}^{K}|\hat{a}_{k,j}|$;
 }
 
 Sort features with decreasing order of $imp$: $imp(x_{j_1})\geq\ldots\geq imp(x_{j_p})$
 
\KwOut{Ordered list of features $x_{j_1},\ldots,x_{j_p}$}
\label{alg3}
\end{algorithm}
\begin{figure}
\begin{center}$
\begin{array}{c}
\includegraphics[scale=0.25]{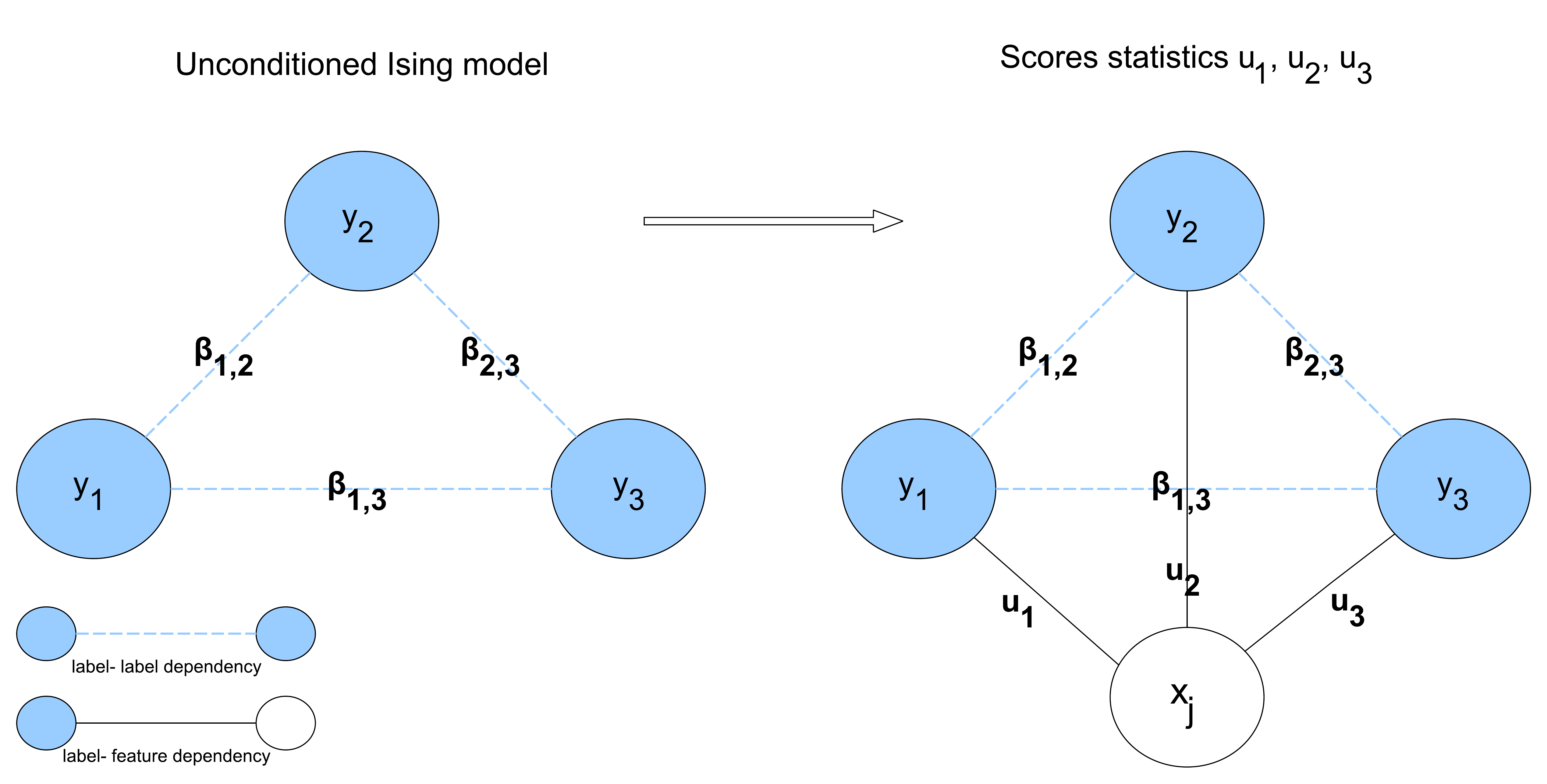}
\end{array}$
\end{center}
\caption{Example scheme of feature importance assessment for $3$ labels.}
\label{fig0}
\end{figure}

\subsection{Feature selection}
\label{Feature selection}
Feature ranking procedures, presented above, allow to order features with respect to their importances, i.e. they produce an ordered list of features $x_{j_1},\ldots,x_{j_p}$, where $x_{j_1}$ is the most relevant feature, whereas $x_{j_p}$ is the least relevant feature. Below we describe how we choose the final subset of features based on the ranking.
Assume that we have multi-label classifier $C(x_1,\ldots,x_p)$ which takes as an input features $x_1,\ldots,x_p$ and returns multi-label output. This classifier is used as a final model. In simulation experiments, classifier chains \citep{Readetal2011, Dembczynskietal2010} were used as a final classifier $C(\cdot)$.     
 First, we split our data into training and validation sets (in simulation experiments: $70\%$ for training and $30\%$ for validation). Training data is used to obtain ranking of features: $x_{j_1},\ldots,x_{j_p}$ and build classifiers. 
We train classifiers $C(x_{j_1}),C(x_{j_1},x_{j_2}),\ldots,C(x_{j_1},\ldots,x_{j_L})$, where $L<p$
and choose a subset $\{x_{j_1},\ldots,x_{j_s}\}$ ($s\leq L$), for which classifier $C(x_{j_1},\ldots,x_{j_s})$ achieves the maximal value of some evaluation measure calculated on validation set. In simulation experiments we use subset accuracy as an evaluation measure. Observe that the final classifier $C(\cdot)$ is built on subsets of features whose size does not exceed $L$ (we set $L=0.2p$). This is very natural approach in the case of large number of possible features $p$ and when the final classifier $C(\cdot)$ requires a significant computational effort (and thus cannot be easily trained on all possible features). Moreover, usually in real data only a small set of features influences the values of labels.     
In the desired feature ranking, the relevant features should precede the spurious ones, and in a consequence the final classifier is built on the subset of most relevant features. When the ranking of features is poor (i.e. there are many spurious features on the top of the list), the resulting classifier will perform poorly. 

\section{Experimental results}
\label{Experimental results}
In this section we evaluate the effectiveness of the proposed methods by comparing its performance against conventional FR methods. We consider two versions of our method: the first one is based on the Ising model with constant interactions (denoted by \textit{ising+score}); the second one is based on the Ising model with feature-dependent interactions (\textit{ising inter+score}). Moreover we consider a method which is  based on the Ising model and $l_1$ regularization (\textit{ising+l1}).
We use two state-of-the-art methods based on BR transformation combined with chi squared statistic (\textit{br chi2}) and information gain (\textit{br ig}). 
In addition we use two conventional methods based on LP transformation combined with chi squared statistic (\textit{lp chi2}) or information gain (\textit{lp ig}). We also experimented with very simple OneR filter \citep{Holte1993}, but the performance was disappointing, so we do not present the results here. 
We carried out experiments on both artificial and real datasets.

\subsection{Evaluation measures}
\label{Evaluation measures}
To evaluate the performance of feature ranking methods on artificial data, we use a form of ROC curves, constructed in the following way (the similar evaluation can be found in \cite{Chengetal2014}).
 Let $i_1,\ldots,i_p$ be the ranking of features from given feature ranking method (where $i_1$ corresponds to a feature recognized as the most significant by an algorithm) and $t$ be a set of true relevant features. Let $TPR(k):=|\{i_1,\ldots,i_k\}\cap t|/|t|$, and $FPR(k):=|\{i_1,\ldots,i_k\}\setminus t|/|t^C|$, where $|\cdot|$ is set cardinality and $t^C$ is set complement.
So, $TPR(k)$ indicates how many relevant features are among top $k$ ones whereas $FPR(k)$ indicates how many redundant features are among top $k$ ones.
Now, ROC curve is defined as $(FPR(k),TPR(k))$, $k=1,\ldots,p$. Observe that $AUC=1$ corresponds to perfect ordering of features, i.e. all relevant features precede spurious ones in the ranking. On the other hand $AUC\approx 0.5$ corresponds to random ordering of features. Each curve is smoothed over $20$ simulations.
This type of evaluation is used to provide an attractive visualization. Note that our ROC curves differ from standard ROC curves used to evaluate the performance of classification models, although the idea is very similar. Our ROC curves are used to evaluate the quality of rankings, not the classification performance.

For real data, we cannot produce ROC curves described above as the relevant features are not known. Thus, in the case of real data, we use standard evaluation measures, described below.
Let $\hat{\y}=(\hat{y}_1,\ldots,\hat{y}_K)^T$ be a vector of predicted labels and $\y=(y_1,\ldots,y_K)^T$ be a vector of true labels.
We consider the following evaluation measures
\begin{equation*}
\textrm{Subset accuracy}(\y,\hat{\y})=I[\y=\hat{\y}],
\end{equation*}
\begin{equation*}
\textrm{Hamming measure}(\y,\hat{\y})=\frac{1}{K}\sum_{k=1}^{K}I[y_k=\hat{y}_k],
\end{equation*}
\begin{equation*}
\textrm{Jaccard measure}(\y,\hat{\y})=\frac{\sum_{k=1}^{K}I(y_k=1 \textrm{ and } \hat{y}_k=1)}{\sum_{k=1}^{K}I(y_k=1 \textrm{ or } \hat{y}_k=1)}
\end{equation*}
The measures are averaged over all instances in test set.
The higher the above measures, the better the performance. The measures demonstrate different aspects of multi-label classification performance. Subset accuracy corresponds to subset $0-1$ loss and measures the correctness of joint prediction for all labels; Hamming measure corresponds to Hamming loss and measures averaged number of correct predictions; Jaccard measure indicates how many labels are correctly predicted as 1 among
those equal 1 or predicted as 1.

\subsection{Artificial data}
\label{Artificial data}
\subsubsection{Correct specification}
The first two datasets are generated under correct specification, i.e. we generate data from the Ising model.
The data generation scheme is as follows. We fix the dimension of features $p=50$, sample size $n=1000$, number of labels $K=10$. Features are generated independently from Gaussian distribution with zero mean and identity covariance matrix.
Labels are generated from the following Ising model
\begin{equation} 
\label{ising3}
P(y_1,\ldots,y_K|\x)=\frac{1}{N(\x)}\exp\left[\sum_{k=1}^{K}\ba_k^{T}\x y_k+\sum_{k<j}\beta_{k,j}y_ky_j+\sum_{k<j}\bb_{k,j}^{T}\x y_ky_j\right],
\end{equation}
where $\ba_k=(a_{k,1},\ldots,a_{k,p})^{T}$ and $\bb_{k,j}=(b_{k,j,1},\ldots,b_{k,j,p})^{T}$ are $p$-dimensional parameter vectors. Model (\ref{ising3}), from which we generate data, incorporates all features $\x=(x_1,\ldots,x_p)^{T}$ simultaneously, not only single feature $x$ as in (\ref{ising1}) and (\ref{ising2}). 
We consider the following two settings.

\textit{ArtData1} Let $t=\{1,\ldots,10\}$ be a set of true relevant features. 
We set  $a_{k,s}=0.2$, for $s\in t$ and $a_{k,s}=0$ for $s\notin t$; $\beta_{k,j}=0.1$; $b_{1,2,s}=b_{2,1,s}=0.2$, for all $s\in t$, $b_{1,2,s}=b_{2,1,s}=0$, for $s\notin t$ and $b_{k,j,s}=0$, for $k,j\notin \{1,2\}$. 

\textit{ArtData2} Let $t=\{1,\ldots,10\}$ be a set of true relevant features. 
We set $a_{k,s}=0$ for all $s$; $\beta_{k,j}=0.1$; $b_{1,2,s}=b_{2,1,s}=0.2$, for all $s\in t$, $b_{1,2,s}=b_{2,1,s}=0$, for $s\notin t$ and $b_{k,j,s}=0$, for $k,j\notin \{1,2\}$.

In both datasets the first $10$ features are significant. In \textit{ArtData2}, significant features do not influence the labels directly but only the interactions between labels, which makes identification of them much more challenging. 
Given feature vector for $i$-th observation $\X^{(i)}$ and parameters defined  above, we use Gibbs sampling to generate labels, where we iteratively generate $\Y^{(i)}_k$, $k=1,\ldots,K$ from Bernoulli distribution with probability $P(\Y^{(i)}_k=1|\X^{(i)},\Y^{(i)}_{-k})$ and take the last value of the sequence. The number  of repetitions in Gibbs sampling is set to $30$. The above procedure is repeated for all $i=1,\ldots,n$.

\begin{figure}
\begin{center}$
\begin{array}{cc}
\includegraphics[scale=0.45]{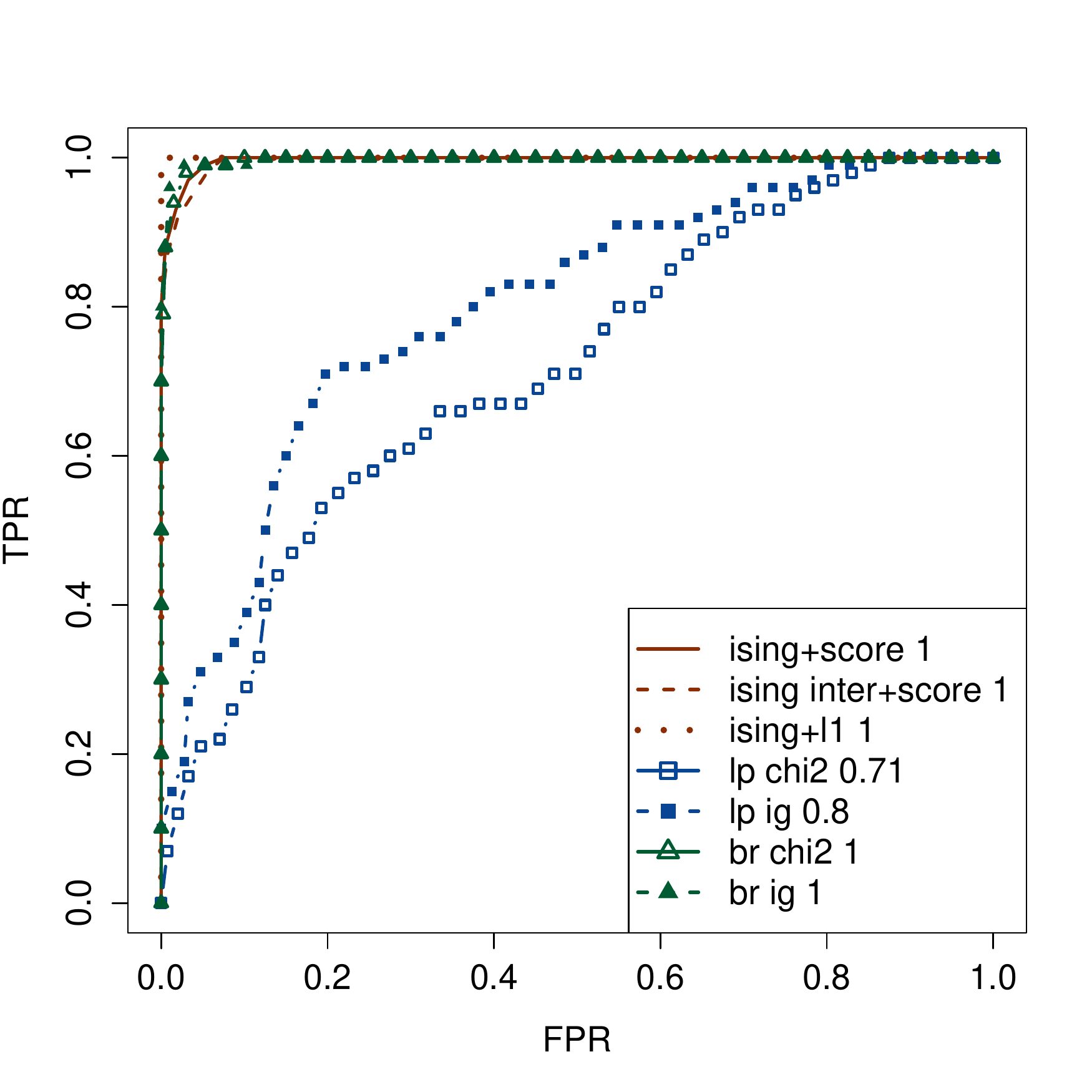} &
\includegraphics[scale=0.45]{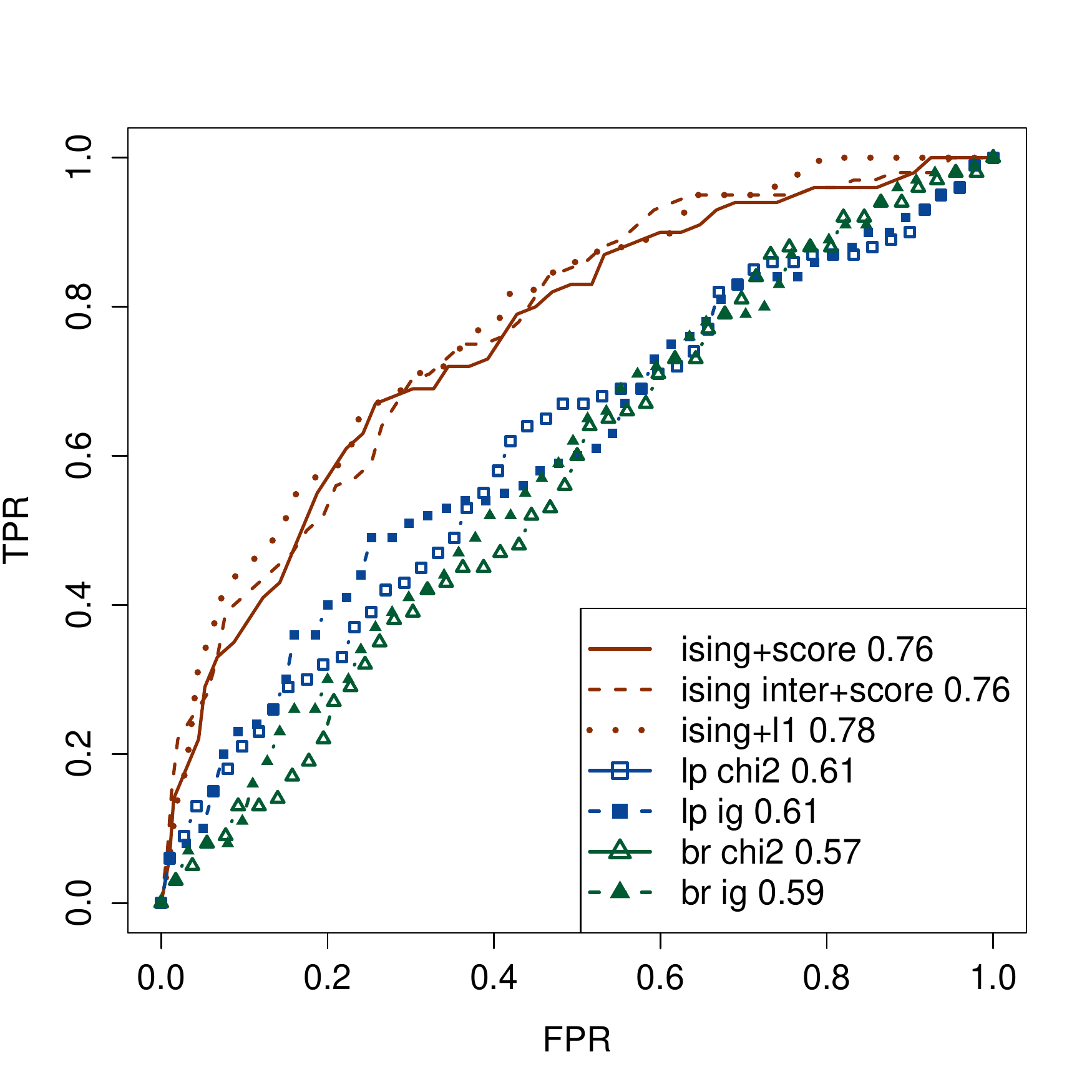} \\
\textrm{\textit{ArtData1}} & \textrm{\textit{\textit{ArtData2}}} \\
\end{array}$
\end{center}
\caption{ROC curves for \textit{ArtData1} and \textit{\textit{ArtData2}}. Numbers in legend correspond to AUC
($\textrm{AUC}=1$ corresponds to perfect ordering of features).}
\label{fig1}
\end{figure}

\begin{figure}
\begin{center}$
\begin{array}{cc}
\includegraphics[scale=0.45]{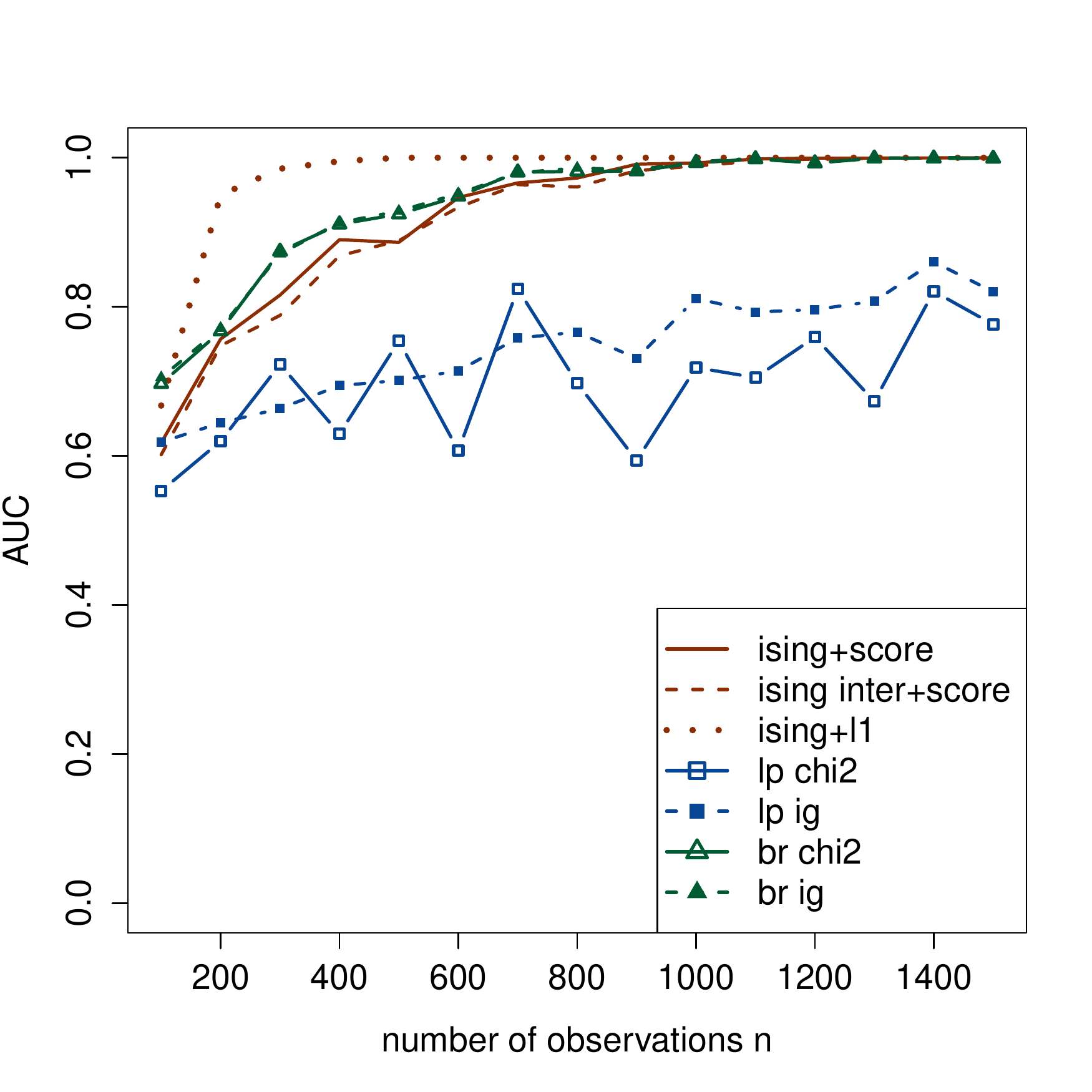} &
\includegraphics[scale=0.45]{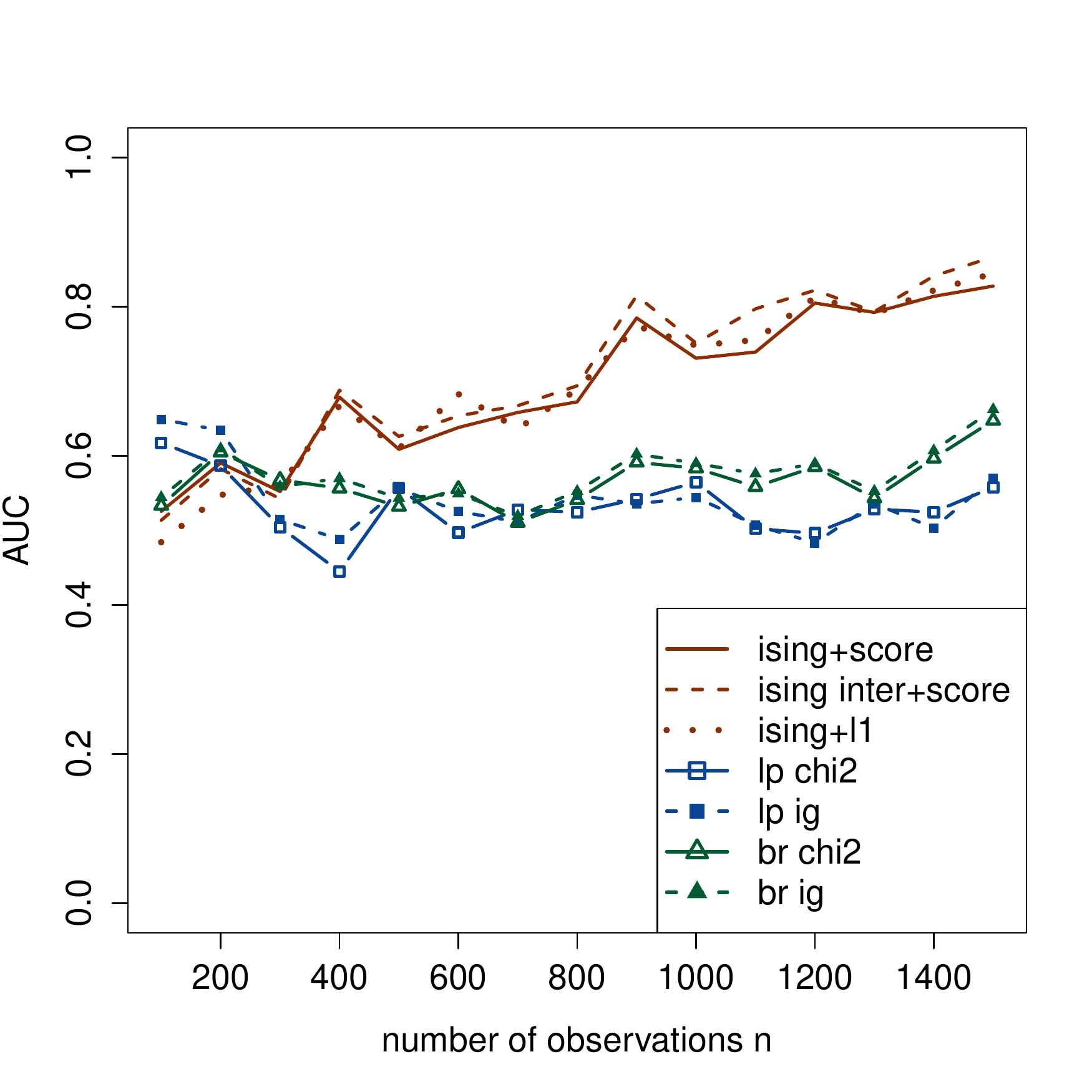} \\
\textrm{\textit{ArtData1}} & \textrm{\textit{\textit{ArtData2}}} \\
\end{array}$
\end{center}
\caption{AUC vs sample size $n$ for \textit{ArtData1} and \textit{\textit{ArtData2}} ($\textrm{AUC}=1$ corresponds to perfect ordering of features).}
\label{fig2}
\end{figure}
\subsubsection{Incorrect specification}
Obviously, data generation scheme presented in the previous section favours the methods based on the Ising model.
It is interesting to investigate the performance of discussed methods, under incorrect specification, i.e. when data generation scheme is not related to the Ising model. For this purpose we generated two datasets closely related to the ones proposed in \cite{DoquireVerleysen2013}. 
We consider the following two settings.

\textit{\textit{ArtData3}} We draw $5$ features $x_1,\ldots,x_5$ from uniform distribution on the $[0,1]$ interval.
Then we construct features $j=6,\ldots,10$ as follows: $x_6=(x_1-x_2)/2$, $x_7=(x_1+x_2)/2$, $x_8=x_3+0.1$, $x_9=x_4-0.2$ and $x_{10}=2x_5$. Then we add additional $40$ features from the uniform distribution on the $[0,1]$, which are independent from $x_1,\ldots,x_{10}$.
The multi-label output is build as follows 
\begin{equation}
\label{multioutput}
 \begin{cases} 
      y_1=1 & \textrm{if } x_1>x_2 \\
      y_2=1 & \textrm{if } x_4>x_3 \\
      y_3=1 & \textrm{if } y_1+y_2=1 \\
      y_4=1 & \textrm{if } x_5>0.8 \\
      y_k=0 & \textrm{otherwise } (k=1,\ldots,4)
   \end{cases}
\end{equation}

\textit{\textit{ArtData4}}
We draw $5$ features $x_1,\ldots,x_5$ from uniform distribution on the $[0,1]$ interval.
Let $\epsilon$ be drawn from Gaussian distribution with $0$ mean and standard deviation equal to $0.3$.
We construct features $j=6,\ldots,10$ as $x_6=x_1+\epsilon$,
$x_7=x_2+\epsilon$,
$x_8=x_3+\epsilon$,
$x_9=x_4+\epsilon$,
$x_{10}=x_5+\epsilon$.
Then we add additional $40$ features from uniform distribution on the $[0,1]$, which are independent from $x_1,\ldots,x_{10}$.
The multi-label output is build as follows 
\begin{equation}
\label{multioutput}
 \begin{cases} 
      y_1=1 & \textrm{if } x_1>x_2+\epsilon \\
      y_2=1 & \textrm{if } x_4>x_3+\epsilon \\
      y_3=1 & \textrm{if } y_1+y_2=1 \\
      y_4=1 & \textrm{if } x_5+\epsilon>0.8 \\
      y_k=0 & \textrm{otherwise } (k=1,\ldots,4)
   \end{cases}
\end{equation}

For \textit{ArtData3} features $t=\{1,\ldots,5,6,8,9,10\}$ are relevant, whereas for \textit{ArtData4}, $t=\{1,\ldots,10\}$.
\textit{ArtData3} was considered in \cite{DoquireVerleysen2013}, section 5.1. \textit{ArtData4} is modification of \textit{ArtData3}, in which  some noise $\epsilon$ is introduced. 
To make the task more challenging we increased the total number of features from $15$ \citep{DoquireVerleysen2013} to $p=50$ and decreased the number of observations from $n=1000$ to $n=100$.
Observe that in \textit{ArtData3} some features can by replaced by others, e.g. $y_1$ is determined by $2$ features: $x_1$ and $x_2$ or by a single feature $x_6$.

\begin{figure}
\begin{center}$
\begin{array}{cc}
\includegraphics[scale=0.45]{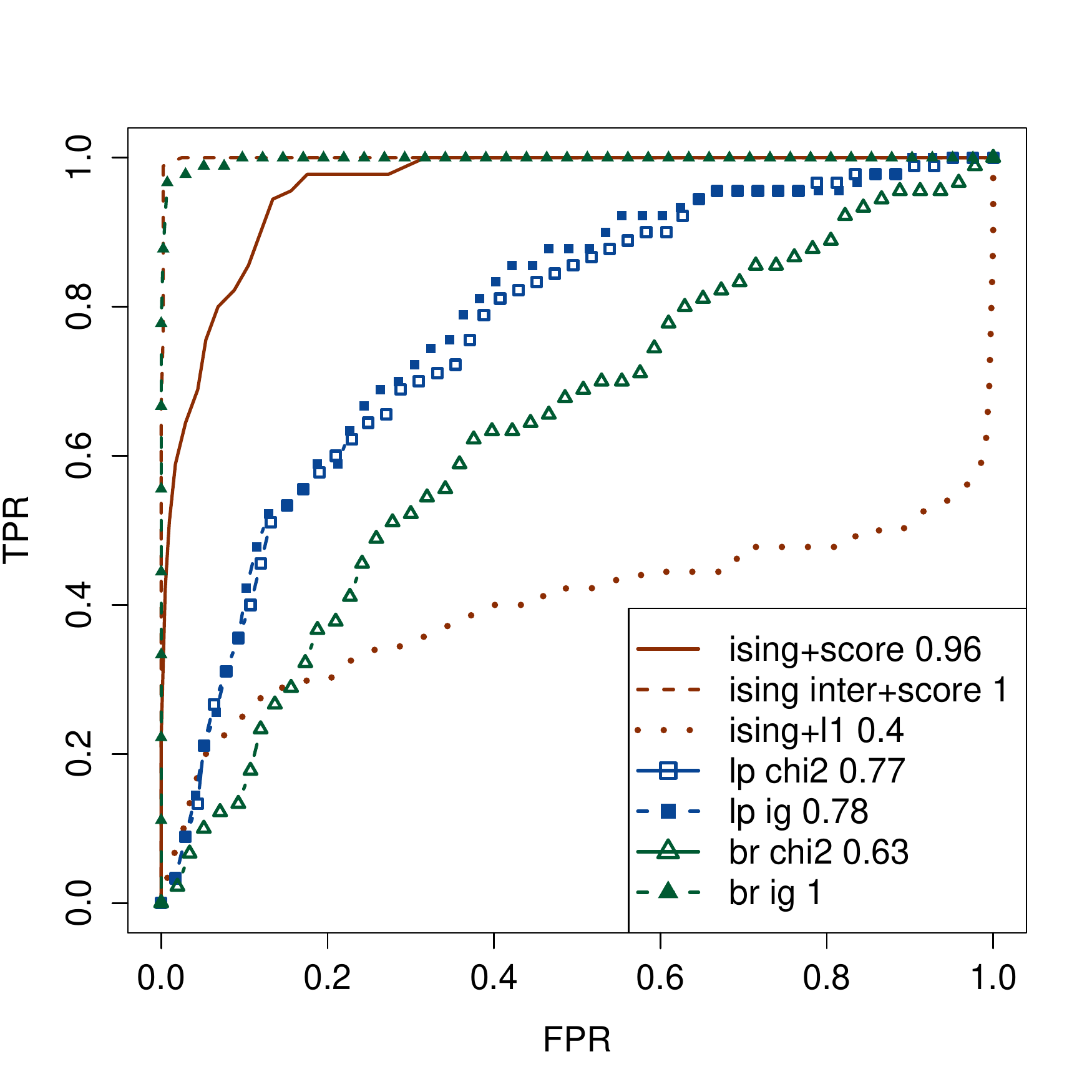} &
\includegraphics[scale=0.45]{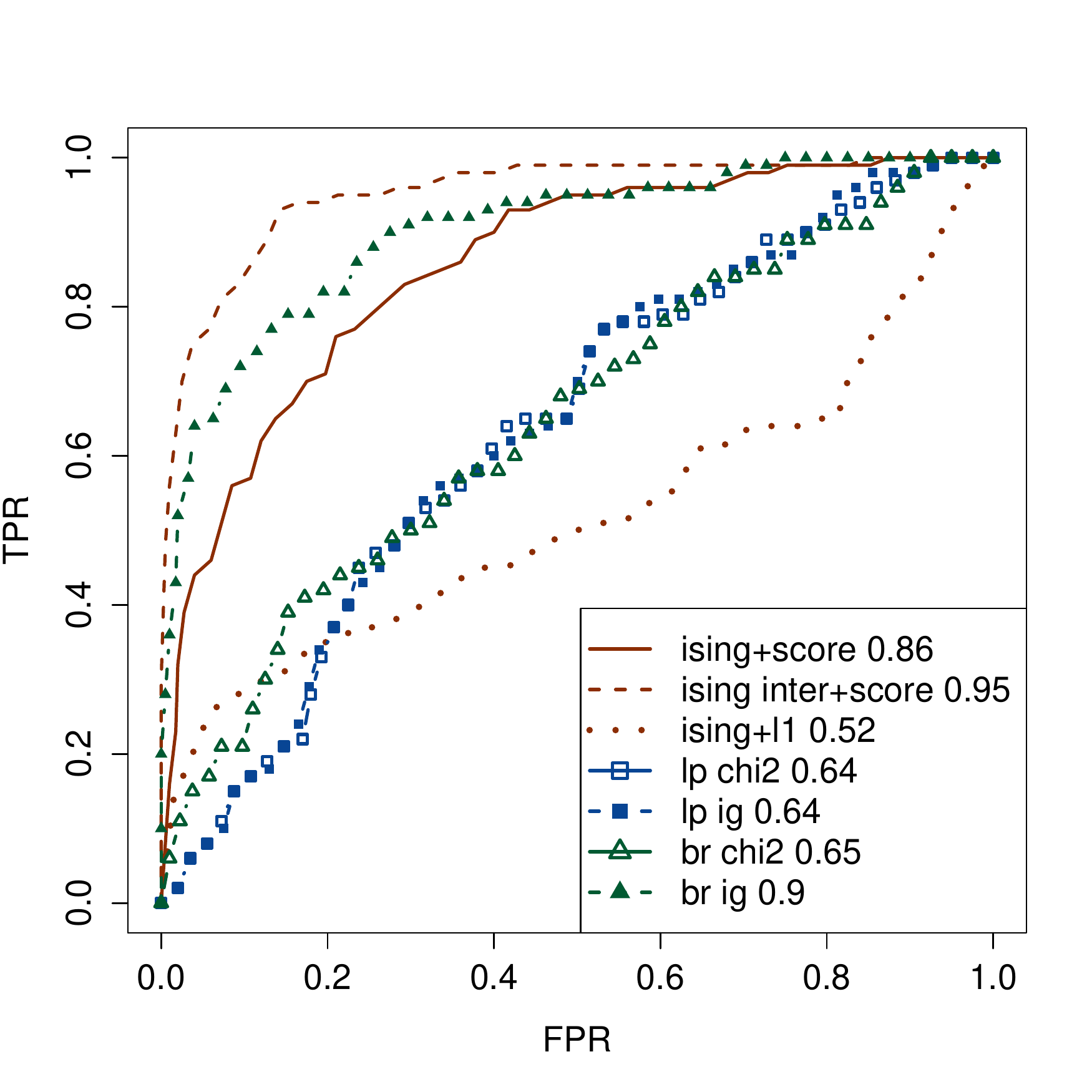} \\
\textrm{\textit{ArtData3}} & \textrm{\textit{ArtData4}} \\
\end{array}$
\end{center}
\caption{ROC curves for \textit{ArtData3} and \textit{ArtData4}. Numbers in legend correspond to AUC ($\textrm{AUC}=1$ corresponds to perfect ordering of features).}
\label{fig1a}
\end{figure}

\begin{figure}
\begin{center}$
\begin{array}{cc}
\includegraphics[scale=0.45]{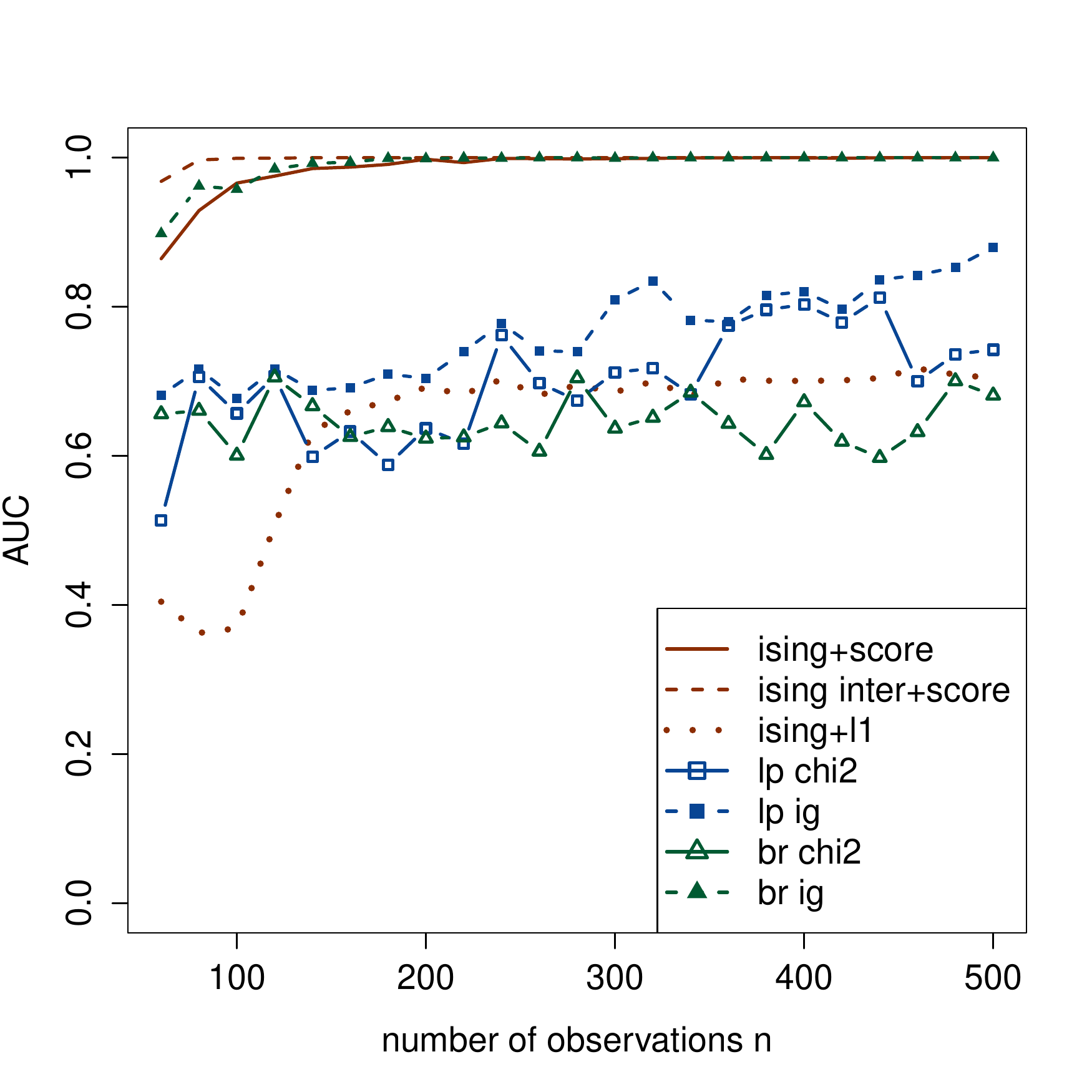} &
\includegraphics[scale=0.45]{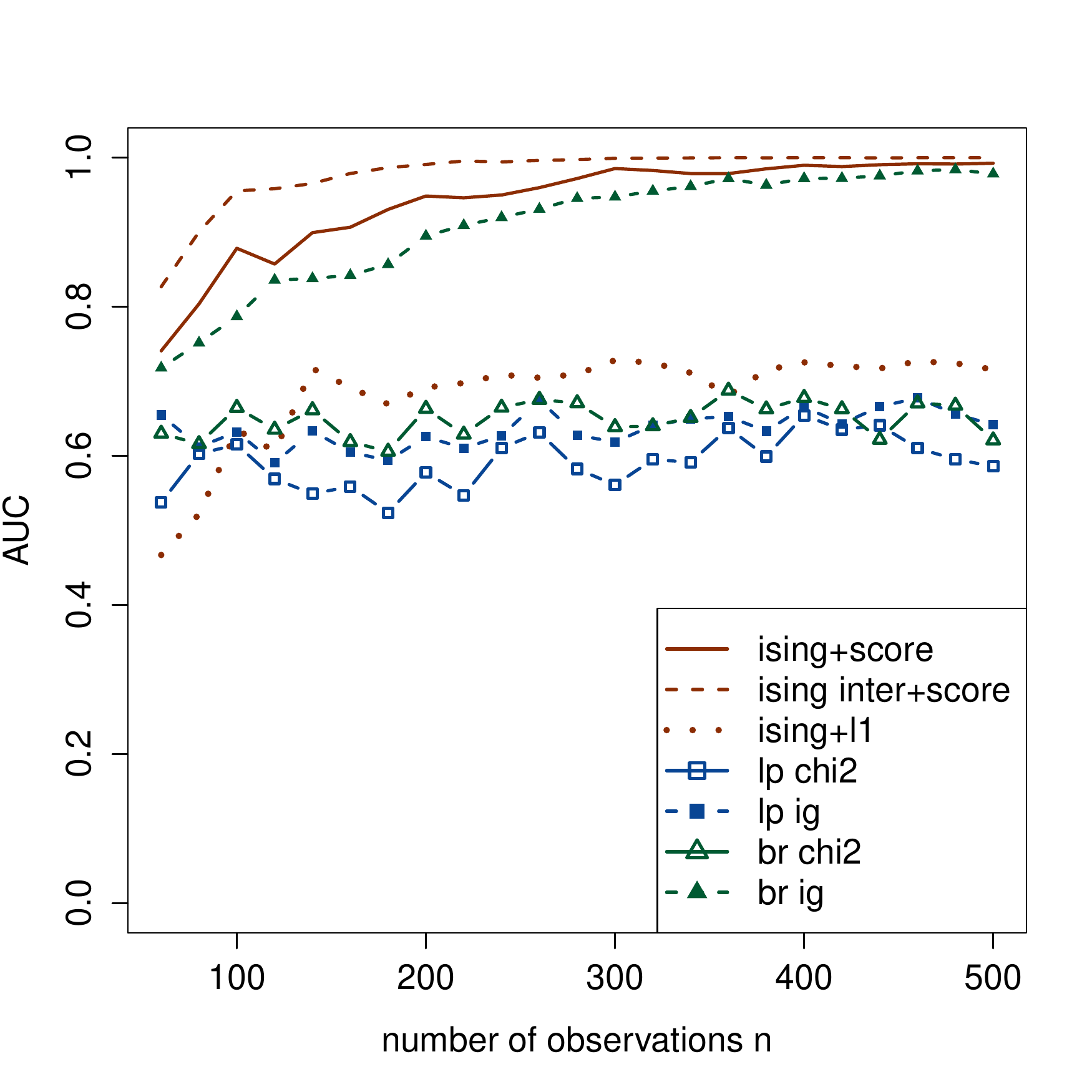} \\
\textrm{\textit{ArtData3}} & \textrm{\textit{ArtData4}} \\
\end{array}$
\end{center}
\caption{AUC vs sample size $n$ for \textit{ArtData3} and \textit{ArtData4} ($\textrm{AUC}=1$ corresponds to perfect ordering of features).}
\label{fig2b}
\end{figure}

\subsection{Experiment 1}
The aim of the first experiment was to study the performance of the proposed feature ranking methods on artificial datasets. In the case of artificial datasets, we can assess the quality of feature ranking methods directly as we know which features are significant, i.e. which features influence the joint probability of labels. For attractive visualization, we use ROC curves described in Section \ref{Evaluation measures}. Recall that desired feature ranking will result in ROC curve significantly above the diagonal and $AUC\approx 1$.

Figure \ref{fig1} shows ROC curves for \textit{ArtData1} and \textit{\textit{ArtData2}}.
It is seen that identification of true relevant features is much more difficult in the case of \textit{\textit{ArtData2}}, which is not surprising as the former one incorporates feature-dependent interactions.
For \textit{ArtData1}, all methods perform well, except methods based on LP.
The proposed methods outperform conventional ones significantly for dataset \textit{\textit{ArtData2}}. 
For \textit{ArtData1}, \textit{lp ig} works slightly better than \textit{lp chi2}, which agrees with the conclusions of other authors \citep{DoquireVerleysen2013}.
In addition, we investigated an effect of data size, i.e. how number of observations affects ranking of features. Figure
\ref{fig2} shows AUC with respect to the sample size $n$. The proposed methods \textit{ising+score} and \textit{ising inter+score} show good performance.
For \textit{ArtData1}, AUC increases with sample size for all methods, but the results for \textit{lp chi2} and \textit{lp ig} are less stable. It is seen that for $n$ large enough, the proposed methods find the correct ranking in all simulations. 
For \textit{ArtData2}, methods based on BR and LP perform poorly, even for large $n$. AUC for proposed methods increases, but a rate of growth is smaller than for \textit{ArtData1}.

Figure \ref{fig1a} shows ROC curves for \textit{ArtData3} and \textit{ArtData4}. It is seen that \textit{ArtData4} is more challenging than \textit{ArtData3} for all considered methods. For both datasets, \textit{ising inter+score} outperforms other approaches, \textit{br ig} is second best, whereas methods based on LP perform poorly. Surprisingly, \textit{ising+l1} works poorly for these two datasets. Effect of sample size is shown on Figure \ref{fig2b}. Large amount of data facilitates the task in the case of \textit{ArtData3}. For \textit{ArtData4}, the performance of LP does not improve even for large $n$.

\subsection{Real data}
We experimented with datasets from different applications.
Some datasets, considered in experiments, are publicly available at \url{http://mulan.sourceforge.net/datasets-mlc.html}.
Those are: scene, yeast, genbase, mediamill, medical, nus-wide, eurlex-dc and CAL500.
We also consider Twitter dataset, analysed in \cite{PrzybylaTeisseyre2015}.
The goal was to analyse a collection of tweets in English and discover its author’s gender, age and personality traits: extraversion, stability, agreeableness, conscientiousness and opennes. Since the original target variables are not binary, we created $7$ binary labels using the original target variables in the following way. We set $y_1=1$ if $gender=male$; $y_2=1$ if $age\leq 34$; $y_i=1$, for $i=3,4,5,6,7$ if the values of target variables extraversion, stability, agreeableness, conscientiousness and opennes are greater than their medians, respectively. The details of the data sets are summarized in Table \ref{tab1}. 
As the proposed methods (in particular \textit{ising inter+score}) are recommended for moderate number of labels, we limited the number of labels to $50$, for all datasets, by taking the most frequent ones. The number of features ranges from $55$ to $5000$. 

\begin{table}
\begin{tabular}{ccccc}
\hline
Dataset&Domain&$\#$observations ($n$)&$\#$features ($p$)&$\#$labels ($K$)\\
\hline
scene& 	images& 	2407& 	 	294& 	6\\
yeast& 	biology& 	2417& 	 	103& 	14 	\\
genbase 	&biology& 	662& 	 	1186& 	27\\
bibtex      &text& 	7395& 	 	1836& 	50\\
mediamill 	&video& 	10000& 	 	120& 	50\\
medical 	&video& 	978& 	 	1449 	& 	45\\
nus-wide	&video& 	10000& 	 	500& 	50\\
eurlex-dc 	&video& 	10000& 	 	5000& 	50\\
CAL500 	&video& 	502& 	 	68& 	50\\
Twitter 	&text& 	152& 	 	55& 	7\\
\hline
\end{tabular}
\caption{Basic statistics for the benchmark datasets.}
\label{tab1}
\end{table}

\subsection{Experiment 2}
The aim of the second experiment was to study the performance of the proposed feature ranking methods  on real-world datasets.
For real-world datasets, relevant features are not known in advance as they were for artificially built datasets. 
Thus the quality of feature ranking methods cannot be evaluated directly but can be measured by the performance of a classification model based on selected features.
We use feature selection procedure described in Section \ref{Feature selection}.
To assess the quality of the considered methods we use measures described in Section \ref{Evaluation measures}: Subset accuracy, Hamming measure and Jaccard measure. 
We also calculate the above measures for artificial datasets, described in Section \ref{Artificial data}.
 
The considered splitting  of the samples into training and test sets are the ones suggested on the website of the Mulan project. The training set is used to obtain ranking of features and than to build a multi-label classification model. For artificial datasets and Twitter dataset we randomly split the original data into training set $50\%$ and testing set $50\%$.
The validation set, needed to choose the final subset of features (see Section \ref{Feature selection}), is separated from the training set. The above performance measures are calculated on the test set. As a final classification model we use Classifier Chains \citep{Readetal2011, Dembczynskietal2010} with logistic model as a base learner.
Our choice is motivated by the fact that classifier chains are among the most frequently used and successful methods in multi-label learning (see \cite{Teisseyre2015} for theoretical properties of classifier chains combined with logistic regression).
We also experimented with nearest neighbour method, but the results were worse, even when the classifier was built using all possible features. Thus the results for nearest neighbour method are not presented.

The results are presented in Tables \ref{tab2}, \ref{tab3} and \ref{tab4}.
 We do not show the results for \textit{ArtData1} and \textit{ArtData2} as the performance of classifier chains was very poor in these cases, for all considered rankings. The reason is that these two datasets were too difficult for the final classifier, even when the final classifier was trained using all possible features.   
  Numbers printed in bold pertain to maximal values in rows (the winning method). The last row contains ranks, averaged over all datasets.
Looking at the averaged ranks for Hamming and Jaccard measures, the proposed methods outperform the conventional ones, although the differences between measures are quite small. For Subset measure, \textit{ising+l1} has the highest averaged rank. Surprisingly, \textit{br ig} outperforms \textit{lp ig}, which can be a consequence of the large number of classes produced by LP transformation. Probably LP transformation combined with some pruning strategies would improve the results.
 
To analyse the results thoroughly, we followed the two-step statistical procedure recommended in \cite{Demsar2006}. In the first step we use the Friedman test of the null hypothesis that all methods have equal performance. Friedman test is based on averaged ranks. When null hypothesis is rejected a post-hoc test is used to compare methods in a pairwise way. We use Conover post-hoc test \citep{Conover1980}. Friedman test suggests that there are significant differences (at a standard significance level $0.05$) between methods for Hamming measure (p-value=$0.0009$) and Subset measure (p-value=$0.0002$), whereas the differences  for Jaccard measure are not statistically significant (p-value=$0.082$). Thus, we performed the post-hoc tests for Hamming and Subset measures. Results of pairwise comparisons for Hamming and Subset measures are shown in Tables \ref{tab5} and \ref{tab6}. It is seen that there are significant differences between the proposed methods and methods based on LP transformation. The differences between the proposed methods and the ones based on BR transformation are not significant. This may be due to the fact that the number of datasets included in experiments is quite limited from a statistical point of view. The overall performance of the proposed approaches is quite promising.    
    
\begin{table}[ht!]
\centering
\begin{tabular}{llllllll}
  \hline
Dataset & ising+score & ising inter+score & ising+l1 & lp chi2 & lp ig & br chi2 & br ig \\ 
  \hline
scene & 0.839 & \textbf{0.845} & 0.844 & 0.826 & 0.815 & 0.825 & 0.813 \\ 
  yeast & \textbf{0.783} & 0.783 & 0.783 & 0.781 & 0.781 & 0.783 & 0.781 \\ 
  genbase & 0.998 & 0.994 & \textbf{0.999} & 0.968 & 0.996 & 0.996 & 0.997 \\ 
  bibtex & 0.967 & 0.968 & 0.960 & 0.961 & 0.968 & \textbf{0.968} & 0.967 \\ 
  enron & \textbf{0.883} & 0.883 & 0.877 & 0.847 & 0.883 & 0.857 & 0.857 \\ 
  mediamill & 0.869 & 0.869 & 0.870 & 0.868 & 0.868 & \textbf{0.871} & 0.869 \\ 
  medical & 0.964 & 0.975 & 0.976 & 0.974 & 0.974 & 0.969 & \textbf{0.976} \\ 
  nus-wide & 0.930 & \textbf{0.931} & 0.929 & 0.929 & 0.929 & 0.930 & 0.930 \\ 
  eurlex-dc & 0.978 & 0.978 & 0.976 & 0.974 & 0.976 & 0.975 & \textbf{0.979} \\ 
  CAL500 & 0.560 & 0.558 & 0.556 & 0.553 & 0.554 & \textbf{0.564} & 0.564 \\ 
  twitter & 0.665 & \textbf{0.680} & 0.613 & 0.660 & 0.660 & 0.618 & 0.648 \\ 
  ArtData3 & 0.683 & 0.684 & 0.671 & 0.661 & 0.661 & 0.635 & \textbf{0.690} \\ 
  ArtData4 & 0.591 & 0.607 & 0.604 & 0.567 & 0.567 & 0.585 & \textbf{0.620} \\ 
    \hline
  Average rank & 4.92&    \textbf{5.23}&    4.30&    2.07&    2.88&    3.76&    4.81 \\  
   \hline
\end{tabular}
\caption{Hamming measure. The
average rank is the average of the ranks across all data sets. Numbers in bold pertain to maximal values in rows.}
\label{tab2}
\end{table}

\begin{table}[ht!]
\centering
\begin{tabular}{llllllll}
  \hline
Dataset & ising+score & ising inter+score & ising+l1 & lp chi2 & lp ig & br chi2 & br ig \\ 
  \hline
scene & 0.486 & 0.508 & \textbf{0.512} & 0.473 & 0.421 & 0.452 & 0.418 \\ 
  yeast & 0.179 & \textbf{0.183} & 0.179 & 0.178 & 0.180 & 0.176 & 0.177 \\ 
  genbase & 0.960 & 0.915 & \textbf{0.980} & 0.613 & 0.940 & 0.945 & 0.940 \\ 
  bibtex & \textbf{0.534} & 0.510 & 0.410 & 0.426 & 0.516 & 0.517 & 0.506 \\ 
  enron & 0.043 & 0.043 & \textbf{0.057} & 0.016 & 0.041 & 0.017 & 0.012 \\ 
  mediamill & 0.121 & 0.119 & \textbf{0.126} & 0.118 & 0.122 & 0.122 & 0.116 \\ 
  medical & 0.451 & 0.640 & \textbf{0.674} & 0.634 & 0.631 & 0.540 & 0.656 \\ 
  nus-wide & 0.299 & \textbf{0.299} & 0.299 & 0.288 & 0.294 & 0.297 & 0.298 \\ 
  eurlex-dc & 0.585 & 0.579 & 0.543 & 0.495 & 0.524 & 0.512 & \textbf{0.605} \\ 
  CAL500 & 0.000 & 0.000 & 0.000 & 0.000 & 0.000 & 0.000 & 0.000 \\ 
  twitter & 0.066 & \textbf{0.092} & 0.092 & 0.053 & 0.053 & 0.079 & 0.079 \\ 
  ArtData3 & 0.030 & 0.030 & 0.030 & 0.010 & 0.010 & 0.010 & 0.010 \\ 
  ArtData4 & 0.010 & 0.010 & 0.010 & 0.000 & 0.000 & 0.000 & 0.010 \\ 
  \hline
 Average rank & 4.80  &  5.07  &  \textbf{5.61} &   2.30 &   3.34  &  3.38 &   3.46\\ 
   \hline
\end{tabular}
\caption{Subset measure. The
average rank is the average of the ranks across all data sets. Numbers in bold pertain to maximal values in rows.}
\label{tab3}
\end{table}
   
\begin{table}[ht!]
\centering
\begin{tabular}{llllllll}
  \hline
Dataset & ising+score & ising inter+score & ising+l1 & lp chi2 & lp ig & br chi2 & br ig \\ 
  \hline
scene & 0.523 & 0.544 & \textbf{0.545} & 0.497 & 0.454 & 0.485 & 0.451 \\ 
  yeast & 0.472 & 0.471 & 0.471 & 0.471 & \textbf{0.472} & 0.469 & 0.470 \\ 
  genbase & 0.983 & 0.953 & \textbf{0.993} & 0.659 & 0.966 & 0.970 & 0.969 \\ 
  bibtex & \textbf{0.154} & 0.123 & 0.000 & 0.026 & 0.136 & 0.135 & 0.116 \\ 
  enron & 0.333 & 0.327 & 0.307 & 0.331 & \textbf{0.343} & 0.311 & 0.313 \\ 
  mediamill & \textbf{0.419} & 0.415 & 0.411 & 0.414 & 0.418 & 0.419 & 0.412 \\ 
  medical & 0.466 & 0.666 & \textbf{0.717} & 0.684 & 0.643 & 0.556 & 0.670 \\ 
  nus-wide & 0.067 & \textbf{0.073} & 0.064 & 0.048 & 0.060 & 0.064 & 0.068 \\ 
  eurlex-dc & 0.077 & 0.071 & 0.033 & 0.004 & 0.016 & 0.002 & \textbf{0.100} \\ 
  CAL500 & \textbf{0.382} & 0.379 & 0.380 & 0.378 & 0.377 & 0.382 & 0.378 \\ 
  twitter & 0.600 & \textbf{0.615} & 0.577 & 0.589 & 0.589 & 0.583 & 0.591 \\ 
  ArtData3 & 0.507 & 0.507 & 0.495 & 0.481 & 0.481 & 0.431 & \textbf{0.520} \\ 
  ArtData4 & 0.371 & 0.387 & 0.429 & 0.371 & 0.371 & 0.390 & \textbf{0.430} \\ 
\hline  
 Average rank &     \textbf{5.23}&    4.76&    3.92&    2.88&    3.65&    3.30&    4.23    \\
   \hline
\end{tabular}
\caption{Jaccard measure. The
average rank is the average of the ranks across all data sets. Numbers in bold pertain to maximal values in rows.}
\label{tab4}
\end{table}

\begin{table}[ht!]
\centering
\begin{tabular}{llllllll}
  \hline
 & ising+score & ising inter+score & ising+l1 & lp chi2 & lp ig & br chi2 \\ 
  \hline
ising inter+score & 1.000 &  &  &  &  &  \\ 
  ising+l1 & 1.000 & 0.885 &  &  &  &  \\ 
  lp chi2 & 0.000 & 0.000 & 0.002 &  &  &  \\ 
  lp ig & 0.007 & 0.001 & 0.142 & 1.000 &  &  \\ 
  br chi2 & 0.438 & 0.128 & 1.000 & 0.042 & 0.903 &  \\ 
  br ig & 1.000 & 1.000 & 1.000 & 0.000 & 0.012 & 0.636 \\ 
   \hline
\end{tabular}
\caption{P-values of post-hoc Conover test, used to compare all classifiers against each other with respect to Hamming measure.}
\label{tab5}
\end{table}
\begin{table}[ht!]
\centering
\begin{tabular}{llllllll}
  \hline
 & ising & ising inter+score & ising+l1 & lp chi2 & lp ig & br chi2 \\ 
  \hline
ising inter+score & 1.0000 &  &  &  &  &  \\ 
  ising+l1 & 0.6638 & 1.0000 &  &  &  &  \\ 
  lp chi2 & 0.0001 & 0.0000 & 0.0000 &  &  &  \\ 
  lp ig & 0.0556 & 0.0136 & 0.0004 & 0.2898 &  &  \\ 
  br chi2 & 0.0634 & 0.0162 & 0.0005 & 0.2769 & 1.0000 &  \\ 
  br ig & 0.0882 & 0.0242 & 0.0008 & 0.2151 & 1.0000 & 1.0000 \\ 
   \hline
\end{tabular}
\caption{P-values of post-hoc Conover test, used to compare all classifiers against each other with respect to Subset measure.}
\label{tab6}
\end{table}

\subsection{Computational efficiency}
The proposed procedures (\textit{ising+score} and \textit{ising inter+score}) require fitting $K$ logistic models (using maximum likelihood method) with $K-1$ input features each in the first step. This step is computationally fast for moderate number of labels $K$.
The first step is a price for taking into account labels, when assessing the relevance of features. The second step includes computation of the score statistic, which is very simple in the case of \textit{ising+score}: the most expensive operation is a computation of scalar products between a given feature and appropriately weighted $K-1$ labels (see definitions of $\B(\hthetab_k)$ and $\C(\hthetab_k)$ in Section \ref{Ising model with constant interaction terms}). In the case of \textit{ising inter+score} we compute the score statistic for a given variable and for products of a given variable with $K-1$ labels. This requires much more operations and can be seen as a price for taking into account feature-dependent interactions  between labels. The third procedure \textit{ising+l1} is the most computationally expensive as it requires fitting $K$ logistic models with $K+p-1$ input features, each, using $l_1$ regularization. To solve the problem we use Cyclic Coordinate Descent (CCD) algorithm proposed by \cite{FriedmanHastieTibshirani2010}. CCD iterates over all $p+K-1$ variables until convergence, the maximal number of iterations in our experiment is set to $100$.
Figure \ref{fig3} shows how the computational time depends on the number of features $p$ (a) and the number of labels $K$ (b). The experiment was carried out on Work Station with Intel Core i5-3220M CPU, 2.60GHz, 12 GB RAM.
All considered methods have been implemented by us in R language; the only exception is an information gain, taken from R package \texttt{FSelector} \citep{FSelector}.
We generated artificial data in such a way that features were drawn from standard Gaussian distribution whereas labels were generated from binomial distribution, number of observations was $n=500$. The curves are smoothed over $5$ simulations. In the case of Figure \ref{fig3} (a) we set $K=5$ and in the case of Figure \ref{fig3} (b) we set $p=50$.
Figure \ref{fig3} (a) indicates that all methods depend linearly on the number of features, except \textit{ising+l1}, which is a price for incorporating all features simultaneously. 
Computational times are larger for methods based on information gain, which is not a surprise as estimation of information gain is more challenging than computation of the chi-squared statistic or the score statistic. The proposed method \textit{ising+score} is among the fastest ones. 

Figure \ref{fig3} (b) indicates that for \textit{ising inter+score} the dependence between computational time and number of labels is quadratic, which is obvious as this method takes into account feature-dependent interactions between labels. Thus this method can be recommended for limited number of labels.

Finally, let us mention that the proposed methods as well as the conventional ones can be computed in parallel easily (the parallel versions were used in \textit{Experiment 1} and \textit{Experiment 2}).

\begin{figure}
\begin{center}$
\begin{array}{cc}
\includegraphics[scale=0.45]{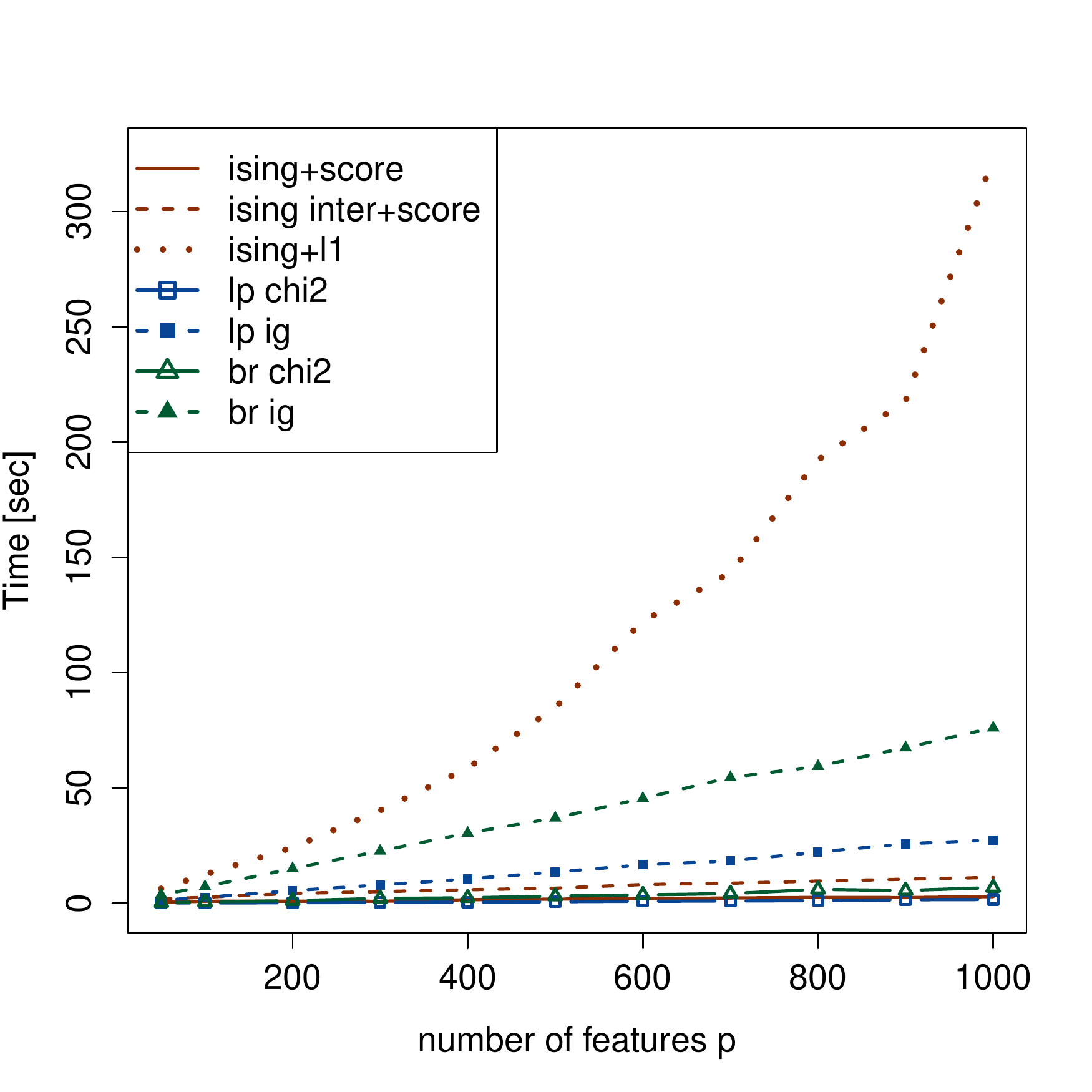} &
\includegraphics[scale=0.45]{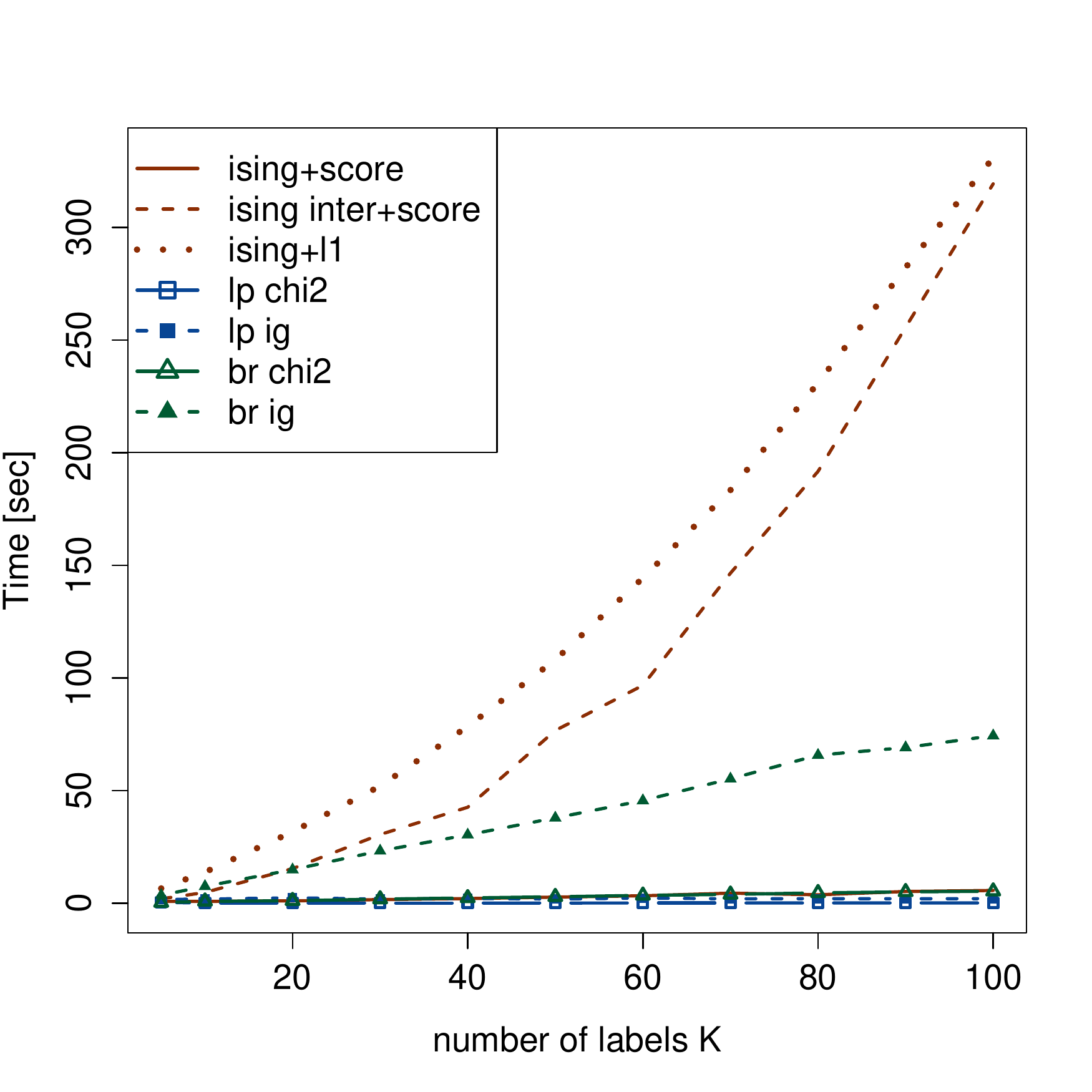} \\
(a) & (b) \\
\end{array}$
\end{center}
\caption{Computational times (in seconds) for considered FR procedures with respect to the number of features $p$ (a) and number of labels $K$ (b).}
\label{fig3}
\end{figure}

\section{Conclusions and future work}
\label{Conclusions}
In this paper we propose a novel method for feature ranking in the multi-label setting.
The method consists of two steps. In the first step we fit the Ising model using only labels. In the second step
we test how much adding a single feature affects the initial network. It is shown that the first step can be performed simply by fitting $K$ logistic models. The second step is based on the score statistic, which is very efficient in this case and allows to test a significance of added features very quickly which is crucial for FR methods. The final feature importance measure is based on averaged values of score statistics. We provide theoretical justification of the Ising model and the score statistic.
We also consider FR procedure based on fitting the Ising model using $l_1$ regularized logistic regressions. This version incorporates all features simultaneously, but it is computationally expensive for large number of labels.
The experiments carried out on artificial and real data show that the proposed methods can outperform the conventional ones. Thus, they can be recommended, especially for datasets with moderate number of labels and large number of features. 

Future work should include generalization of the proposed approach to more general Markov Networks. In particular one can consider a generalized Ising model in which some non-linear functions of features are used instead of linear combinations. The major problem associated with more general Markov Networks is how to estimate the parameters and how to test the significance of features efficiently.

In our procedure we test the significance of feature $x_j$ in model $y_k\sim \y_{-k},x_j$ using logistic regression. This suggests that other classification models can be used, e.g. decision trees which allow to discover non-linear dependencies. Note however that the main problem with decision trees (and some other classification models) is how to verify the significance of $x_j$ in model $y_k\sim \y_{-k},x_j$, possibly without refitting the model when adding $x_j$ (as in our procedure). The modification of the score statistic $u_{k}(x_j)$ would be necessary for other models.

The other interesting question is how to choose the final subset of features having their ordering. Here we use a simple approach based on validation set. It would be worthwhile to have a more sophisticated method, which does not require separating a validation set.

The limitation of many FR methods is that the features are accessed individually and the possible redundancy as well as joint relevance of features is not taken into account. On the other hand the methods, which take into account all features simultaneously (e.g. \textit{ising+l1}), are usually slow for large number of features or labels.
It would be interesting to combine methods which assess the individual relevance of the features (like \textit{ising+score}) with those taking into account all features simultaneously (like \textit{ising+l1}).
This could be done by applying two-step procedure in which \textit{ising+score} is used at first to filter out least significant features and then \textit{ising+l1} is launched on the remaining set of features.

 
\section*{Acknowledgements}
I am grateful to the Associate Editor and anonymous reviewers
for their valuable comments that helped to improve the initial
version of this paper.

Research of Pawe{\l} Teisseyre  was supported by the European Union from resources of the European Social Fund within project 'Information technologies: research and their interdisciplinary applications' POKL.04.01.01-00-051/10-00.

\appendix 
\section{}
\label{Appendix}
\subsection{Proof of (\ref{logodds})}
\label{Appendix A}
For simplicity, we write $x$ instead of $x_j$.
Using the definition of conditional probabilities we can write
\begin{eqnarray*}
&&
\frac{P(y_k=1|x,\y_{-k})}{P(y_k=0|x,\y_{-k})}=
\frac{P(y_k=1,x,\y_{-k})/P(x,\y_{-k})}{P(y_k=0,x,\y_{-k})/P(x,\y_{-k})}=
\frac{P(y_k=1,x,\y_{-k})}{P(y_k=0,x,\y_{-k})}=
\cr
&&
\frac{\exp[a_kx+\sum_{l:l\neq k}a_lxy_l+\sum_{s<l:s,l\neq k}\beta_{s,l}y_sy_l+\sum_{l:l\neq k}\beta_{l,k}y_l]}{
\exp[\sum_{l:l\neq k}a_lxy_l+\sum_{s<l:s,l\neq k}\beta_{s,l}y_sy_l]
}=
\exp[a_kx+\sum_{l:l\neq k}\beta_{l,k}y_l],
\end{eqnarray*}
which ends the proof.
\subsection{Proof of Proposition \ref{Proposition 1}}
\label{Appendix B}
We can write
\begin{eqnarray*}
&&
H_g(\y|x)=-\sum_{\y}g(\y|x)\log(g(\y|x))=
-\sum_{\y}g(\y|x)\log\left[\frac{g(\y|x)}{p(\y|x)}p(\y|x)\right]=
\cr
&&
-KL(g,p)-\sum_{\y}g(\y|x)\log(p(\y|x))\leq 
-\sum_{\y}g(\y|x)\log(p(\y|x)),
\end{eqnarray*}
where $KL(g,p)$ is a Kullback-Leibner divergence between $g$ and $p$ and the last inequality follows from $KL(g,p)\geq 0$ (see e.g. Theorem 8.6.1 in \cite{Cover2006}).
Using the definition of $p$ and the fact that both $p$ and $g$ must satisfy constraints (\ref{constr1}) and (\ref{constr2}), we obtain
\begin{eqnarray*}
&&
-\sum_{\y}g(\y|x)\log(p(\y|x))=
-\sum_{\y}g(\y|x)\left[-\log(Z(x))+\sum_{k=1}^{K}a_kxy_k+\sum_{k<j}(\beta_{k,j}+b_{k,j}x)y_ky_j\right]=
\cr
&&
-\sum_{\y}p(\y|x)\left[-\log(Z(x))+\sum_{k=1}^{K}a_kxy_k+\sum_{k<j}(\beta_{k,j}+b_{k,j}x)y_ky_j\right]=
-\sum_{\y}p(\y|x)\log(p(\y|x)),
\end{eqnarray*}
which ends the proof.

\subsection{Auxiliary facts}
Let us introduce some additional notation. In the following $||\w||$ will denote Euclidean norm of vector $
\w$ and $||\w||_{\infty}$ maximum norm. In addition $\lambda_{j}(\A)$ denotes $j$-th eigenvalue of matrix $\A$, $\lambda_{\min}(\A)$ ($\lambda_{\max}(\A)$) its minimal (maximal) eigenvalue. 

Let $l(\cdot)$ be log-likelihood function based on larger model $y_k\sim\y_{-k},x$. 
Let $\s(\cdot)$ be gradient of $l(\cdot)$. 
Recall that $\Z=(\Y_{-k},\X_j)$ is $n\times K$ matrix.
It is easy to calculate that $\s(\thetab_k)=\Z^{T}(\Y_k-\p(\thetab_k))$,  $\p(\thetab_k)=(p^{(1)}(\thetab_k),\ldots,p^{(n)}(\thetab_k))^{T}$.
Since coordinates of $\s(\thetab_k)$ are sums whose summands are bounded by $L$, it follows from Hoeffding inequality that 
\begin{equation}
\label{Hoefding}
P(||\s(\thetab_k)||_{\infty}>\delta|\Z)\leq K\exp\left[-\frac{2\delta^2}{nL^2}\right],
\end{equation}
for any $\delta>0$.

For logistic regression, Hessian matrix of $l(\cdot)$ is equal $-\I(\cdot)$, where 
$\I(\cdot)=\Z^{T}\W(\cdot)\Z$.

\begin{Lemma}
\label{Lemma1}
Assume that $|(\w-\thetab_k)^{T}\Z^{(i)}|\leq 1$, for some vector $\w\in R^{K}$. Then 
\begin{equation*}
p^{(i)}(\w)(1-p^{(i)}(\w))>e^{-3}p^{(i)}(\thetab_k)(1-p^{(i)}(\thetab_k)).
\end{equation*}
\end{Lemma}
\begin{proof}
Observe that for $\w$ such that $|(\w-\thetab_k)^{T}\Z^{(i)}|\leq 1$ we have
\begin{equation*}
\frac{p^{(i)}(\w)(1-p^{(i)}(\w))}{p^{(i)}(\thetab_k)(1-p^{(i)}(\thetab_k))}=
e^{(\w-\thetab_k)^{T}\Z^{(i)}}\left[\frac{1+e^{\thetab_k^{T}\Z^{(i)}}}{1+e^{\w^{T}\Z^{(i)}}}\right]^{2}\geq
e^{-1}\left[\frac{e^{-\thetab_k^{T}\Z^{(i)}}+1}{e^{-\thetab_k^{T}\Z^{(i)}}+e}\right]^{2}\geq e^{-3}.
\end{equation*}
\end{proof}
Recall that $\hthetab_k$ is an estimator of $\thetab_k$ based on model $y_k\sim\y_{-k}$ in which the last coordinate corresponding to $x_j$ is set to $0$.
\begin{Lemma}
\label{Lemma2}
The following inequality holds
\begin{equation*}
P[l(\thetab_k)-l(\hthetab_k)>e^{-3}\Lambda_{\min}vnd^2/4|\Z]\geq 1-K\exp\left[-\frac{Cn(K+L^2)a_k^{2}}{2H^2}\right],
\end{equation*}
where $d=\frac{|a_k|}{\sqrt{K+L^2}H}$, $H=\max(1,G)$.
\end{Lemma}
\begin{proof}
Define set $A=\{\w: ||\w-\thetab_k||\leq d\}$ and observe that the last coordinate of $\hthetab_k$ is set to $0$, thus $\hthetab_k\notin A$, as $d\leq |a_k|$. Define function $H(\w):=l(\thetab_k)-l(\w)$ and observe that $H(\w)$ is convex (as $l(\cdot)$ is concave) and $H(\thetab_k)=0$. Thus is suffices to show that $H(\w)>e^{-3}\Lambda_{\min}vnd^2/4$ on the boundary of $A$, i.e. for $\w$ such that $||\w-\thetab_k||=d$, with large probability.  

Using Taylor expansion and the fact that Hessian matrix of $l(\cdot)$ is equal to $-\I(\cdot)$, we can write
\begin{equation}
\label{L2_e1}
H(\w)=-(\w-\thetab_k)^{T}\s(\thetab_k)+(\w-\thetab_k)^{T}\I(\bar{\w})(\w-\thetab_k)/2,
\end{equation}
where $\bar{\w}$ is some point in set $A$.

Using Cauchy-Schwarz inequality we have
\begin{equation}
\label{L2_e2}
|(\w-\thetab_k)^{T}\Z^{(i)}|\leq ||(\w-\thetab_k)||\cdot ||\Z^{(i)}||\leq d\sqrt{K+L^2}=\frac{|a_k|}{\max(1,G)}\leq 1.
\end{equation}
It follows from (\ref{L2_e2}), Lemma \ref{Lemma1} and Assumption 2 that
\begin{equation}
\label{L2_e3}
(\w-\thetab_k)^{T}\I(\bar{\w})(\w-\thetab_k)/2\geq 
e^{-3}(\w-\thetab_k)^{T}\I(\thetab_k)(\w-\thetab_k)/2\geq 
e^{-3}d^{2}\Lambda_{\min} vn/2.
\end{equation}
Observe that 
\begin{equation}
\label{L2_e4}
|(\w-\thetab_k)^{T}\s(\thetab_k)|\leq \sqrt{K}||\w-\thetab_k||\cdot||\s(\thetab_k)||_{\infty}\leq \sqrt{K+L^2}d||\s(\thetab_k)||_{\infty}.
\end{equation}
Now using (\ref{L2_e1}), (\ref{L2_e3}) and (\ref{L2_e4}) we can write
\begin{eqnarray*}
&&
P[H(\w)>e^{-3}\Lambda_{\min}vnd^2/4|\Z]\geq
P[-d||\s(\thetab_k)||_{\infty}\sqrt{K+L^2}+e^{-3}\Lambda_{\min}vnd^2/2\geq e^{-3}\Lambda_{\min}vnd^2/4|\Z]\geq
\cr
&&
P\left[||\s(\thetab_k)||_{\infty}\leq \frac{d\Lambda_{\min}vn}{4e^{3}\sqrt{K+L^2}}|\Z\right]\geq
1-K\exp\left[-\frac{d^2\Lambda_{\min}^{2}v^{2}n}{8e^{6}(K+L^2)L^2}\right]=
1-K\exp\left[-\frac{Cn(K+L^2)a_k^{2}}{2[\max(1,G)]^2}\right],
\cr
\end{eqnarray*}
where the last inequality follows from (\ref{Hoefding}). This ends the proof.
\end{proof}

Recall that $v(\hthetab_k)=D(\hthetab_k)-\C(\hthetab_k)\A^{-1}(\hthetab_k)\B(\hthetab_k)$, where $\A, \B, \C, D$ are defined in Section \ref{Ising model with constant interaction terms}.
\begin{Lemma}
\label{Lemma3}
The following inequality holds
\begin{equation*}
v^{-1}(\hthetab_k)\geq\frac{4}{(K+L^2)L^{2}n}.
\end{equation*}
\end{Lemma}
\begin{proof}
First observe that using a definition of Shur complement (see e.g. \cite{Gentle2007}, Section 3.4) we have that $v^{-1}(\hthetab_k)=[\I^{-1}(\hthetab_k)]_{K,K}$, where $[A]_{K,K}$ denotes element in $K$-th row and $K$-th column of matrix $A$. Observe that
\begin{equation*}
[\I^{-1}(\hthetab_k)]_{K,K}\geq \lambda_{\min}(\I^{-1}(\hthetab_k))=\frac{1}{\lambda_{\max}(\I(\hthetab_k))}.
\end{equation*}
Since $p^{(i)}(\hthetab_k)(1-p^{(i)}(\hthetab_k))<0.25$, each element on the diagonal of $\I(\hthetab_k)$ is bounded by $L^{2}n/4$ and thus
\begin{equation*}
\lambda_{\max}(\I(\hthetab_k))\leq \sum_{j=1}^{K}\lambda_{j}(\I(\hthetab_k))=
\sum_{j=1}^{K}[\I(\hthetab_k)]_{j,j}\leq K\max_{j}[\I(\hthetab_k)]_{j,j}\leq \frac{KL^2n}{4}\leq \frac{(K+L^2)L^2n}{4},
\end{equation*}
which ends the proof.
\end{proof}
\begin{Lemma}
\label{Lemma4}
Let $\Lambda_{\min}=\lambda_{\min}(\Z^{T}\Z/n)>0$. Then function $l(\cdot)$ is concave.
\end{Lemma}
\begin{proof}
Note that $\Lambda_{\min}>0$ implies positive definiteness of $\Z^{T}\Z$. 
We have to show that $\I(\w)$ is positive definite for any $\w\in R^{K}$.
For any vectors $\w,\bc\in R^{K}$ we have
\begin{equation*}
\min_{i}p^{(i)}(\w)(1-p^{(i)}(\w))\bc^{T}\Z^{T}\Z\bc\leq \bc^{T}\I(\w)\bc. 
\end{equation*}
Since $\min_{i}p^{(i)}(\w)(1-p^{(i)}(\w))>0$, positive definiteness of $\Z^{T}\Z$ implies positive definiteness of  $\I(\w)$, for any $\w$, which ends the proof.
\end{proof}
\subsection{Proof of Theorem \ref{Theorem 1}}
\label{Proof of Theorem1}
Using Taylor expansion of log-likelihood function we obtain
\begin{equation}
\label{T1_e1}
l(\thetab_k)=l(\hthetab_k)+(\thetab_k-\hthetab_k)^{T}\s(\hthetab_k)-
(\thetab_k-\hthetab_k)^{T}\I(\bar{\thetab}_k)(\thetab_k-\hthetab_k)/2,
\end{equation}
where $\bar{\thetab}_k$ is point on the line segment between $\thetab_k$ and $\hthetab_k$.
Note that the first $K-1$ coordinates of $\s(\hthetab_k)$ are equal zero and thus (\ref{T1_e1}) reduces to
\begin{equation}
\label{T1_e2}
l(\thetab_k)-l(\hthetab_k)=a_{k}s(\hthetab_k)-
(\thetab_k-\hthetab_k)^{T}\I(\bar{\thetab}_k)(\thetab_k-\hthetab_k)/2.
\end{equation}
Now from (\ref{T1_e2}), non-negativity of $(\thetab_k-\hthetab_k)^{T}\I(\bar{\thetab}_k)(\thetab_k-\hthetab_k)/2$ and Lemma \ref{Lemma2} we have
\begin{equation}
\label{T1_e3}
|s(\hthetab_k)|\geq\frac{\Lambda_{\min}vnd^2}{4e^3a_k}=\frac{\Lambda_{\min}vn|a_k|}{4e^3(K+L^2)[\max(1,G)]^{2}},
\end{equation}
with probability given in Lemma \ref{Lemma2}. The assertion of the Theorem follows directly from (\ref{T1_e3}) and Lemma \ref{Lemma3}.



  \bibliographystyle{elsarticle-num} 
  \bibliography{References}


\end{document}